\documentclass[autheryear, times, 9pt]{elsarticle}

\usepackage{amsmath}
\usepackage{amsthm}
\usepackage{amssymb}
\usepackage{bm}
\usepackage{dutchcal}
\usepackage{geometry}
\usepackage{times}
\usepackage{url}

\usepackage{longtable}

\usepackage{graphicx}
\usepackage{epstopdf}

\usepackage{multirow}
\usepackage{subcaption}
\usepackage{verbatim}
\usepackage{enumitem}
\usepackage{color}
\usepackage{slashbox}

\usepackage{float}
\restylefloat{figure}

\usepackage{hyperref}

\usepackage{xcolor}
\hypersetup{
    colorlinks,
    linkcolor={red!50!black},
    citecolor={blue!50!black},
    urlcolor={blue!80!black}
}

\newtheorem{proposition}{Proposition}

\usepackage{natbib}

\hypersetup{
    colorlinks,
    linkcolor={red!50!black},
    citecolor={blue!50!black},
    urlcolor={blue!80!black}
}

\biboptions{authoryear}
\DeclareMathOperator*{\argmin}{argmin}

\DeclareMathOperator{\tr}{tr}
\setcitestyle{notesep={; },round,aysep={},yysep={;}}
\journal{Transportation Research Part C: Emerging Technologies}

\begin{document}
\begin{frontmatter}
\title{Turning Traffic Monitoring Cameras into Intelligent Sensors for Traffic Density Estimation}
\author[address1]{Zijian Hu}
\ead{zijian.hu@connect.polyu.hk}
\author[address1]{William H.K. Lam}
\ead{william.lam@polyu.edu.hk}
\author[address2]{S.C. Wong}
\ead{hhecwsc@hku.hk}
\author[address3]{Andy H.F. Chow}
\ead{andychow@cityu.edu.hk}
\author[address1,address4]{Wei Ma\corref{mycorrespondingauthor}}
\cortext[mycorrespondingauthor]{Corresponding author}
\ead{wei.w.ma@polyu.edu.hk}
\address[address1]{Civil and Environmental Engineering, The Hong Kong Polytechnic University \\ Hung Hom, Kowloon, Hong Kong SAR, China}
\address[address2]{Department of Civil Engineering, The University of Hong Kong \\ Pokfulam, Hong Kong Island, Hong Kong SAR, China}
\address[address3]{Department of Advanced Design and Systems Engineering, City University of Hong Kong \\ Kowloon Tong, Kowloon, Hong Kong SAR, China}
\address[address4]{Research Institute for Sustainable Urban Development, The Hong Kong Polytechnic University \\ Hung Hom, Kowloon, Hong Kong SAR, China}

\begin{abstract}
Accurate traffic state information plays a pivotal role in the Intelligent Transportation Systems (ITS), and it is an essential input to various smart mobility applications such as signal coordination and traffic flow prediction. The current practice to obtain the traffic state information is through specialized sensors such as loop detectors and speed cameras. In most metropolitan areas, traffic monitoring cameras have been installed to monitor the traffic conditions on arterial roads and expressways, and the collected videos or images are mainly used for visual inspection by traffic engineers. Unfortunately, the data collected from traffic monitoring cameras are affected by the {\bf4L} characteristics: {\bf L}ow frame rate, {\bf L}ow resolution, {\bf L}ack of annotated data, and {\bf L}ocated in complex road environments. Therefore, despite the great potentials of the traffic monitoring cameras, the 4L characteristics hinder them from providing useful traffic state information ({\em e.g.}, speed, flow, density). This paper focuses on the traffic density estimation problem as it is widely applicable to various traffic surveillance systems. To the best of our knowledge, there is a lack of the holistic framework for addressing the 4L characteristics and extracting the traffic density information from traffic monitoring camera data. In view of this, this paper proposes a framework for estimating traffic density using uncalibrated traffic monitoring cameras with 4L characteristics. The proposed framework consists of two major components: camera calibration and vehicle detection. The camera calibration method estimates the actual length between pixels in the images and videos, and the vehicle counts are extracted from the deep-learning-based vehicle detection method. Combining the two components, high-granular traffic density can be estimated. To validate the proposed framework, two case studies were conducted in Hong Kong and Sacramento. The results show that the Mean Absolute Error (MAE) in camera calibration is less than 0.2 meters out of 6 meters, and the accuracy of vehicle detection under various conditions is approximately 90\%. Overall, the MAE for the estimated density is 9.04 veh/km/lane in Hong Kong and 1.30 veh/km/lane in Sacramento. The research outcomes can be used to calibrate the speed-density fundamental diagrams, and the proposed framework can provide accurate and real-time traffic information without installing additional sensors.
\end{abstract}
\end{frontmatter}

\section{Introduction}
Accurate real-time and traffic state information is the essential input to the Intelligent Transportation Systems (ITS) with various traffic operation and management tasks, such as ramp metering \citep{alinea, au_alinea, pi_alinea}, perimeter control \citep{gating, pp1}, congestion pricing \citep{cp1, cp2, cp3}, {\em etc}. In recent years, many cities have expended considerable efforts on installing traffic detectors to obtain traffic information. However many ITS applications are still data-hungry. Using Hong Kong as an example, the current detectors ({\em e.g.,} loop detectors) only cover approximately 10\% of the road segments, which is not sufficient to support the network-wide traffic modeling and management framework. How to collect the real-time traffic state information in an accurate, efficient, and cost-effective manner presents a long-standing challenge for not only the research community but also the private sector ({\em e.g.,} Google Maps) and the public agency ({\em e.g.,} Transport Department).

Traffic state information can be categorized into speed, density, and flow data, each of which requires specialized traffic sensors. In general, speed estimation is relatively straightforward. For example, GPS devices can be installed on private or public vehicles (such device-equipped vehicles are known as probe vehicles, PVs) \citep{trafficspeedpv}, and speed information can be estimated accurately  with low PV penetration rates of 4\%-5\%  \citep{traffic_speed_penerating}. PV-based speed estimation has already been applied to real-world such as Google Maps, Uber, {\em etc}. However, traffic flow and density are more challenging to be estimated, which requires a full-penetration observation of a road segment. Given the traffic speed, density and flow follow the ``one computes the other'' characteristics based on the speed-flow-density relationship \citep{NI201651}. 
In this paper, we focus on estimating traffic density on road segments as it is applicable to various traffic surveillance systems around the globe. The proposed framework could potentially be extended to traffic flow estimation using video data which we leave for future research.

Recent years have witnessed great advances in emerging technologies for traffic sensing, and various sensors and devices can be employed to estimate the traffic density on urban roads. A review of existing studies on traffic density estimation is shown in Table~\ref{tab:review_estimation}. 

\begin{table}[h]
    \begin{tabular}{p{0.23\columnwidth}|p{0.35\columnwidth}|p{0.35\columnwidth}}
    \hline 
    \textbf{Sensors} & \textbf{Advantages} & \textbf{Disadvantages} \\ \hline \hline
    {\bf Point sensors}  \citet{NI20163} & 1. Steady data sources for 24/7 monitoring. & 1. Expensive and difficult for massive installation and  maintenance.  \\ \hline
    {\bf VANET}  & 1. No additional hardware required. & 1.Limited accuracy when the penetration rate of PVs is low. \\ \citet{Panichpapiboon08} & 2. Potential data sources covered a large-scale traffic network. & 2. Rare pilot study has been conducted. \\ \hline 
    {\bf UAV} & 1. High flexibility and instant deployment. & 1. Challenging to long-time estimation with large perspective. \\ \citet{Zhu2018, Ke2019} & 2. High fidelity data sources. & 2. Expensive for massive deployment. \\ \hline
    {\bf Traffic monitoring cameras} & 1. Widespread in many cities. & 1. Low data quality leading to potentially inaccurate results. \\ This paper & 2. Steady data sources for 24/7 monitoring. & 2. Owing to privacy concerns, sometimes only images (and not videos) can be acquired. \\ \hline
    \end{tabular}
    \caption{A review of emerging sensors for estimating traffic density.}
    \label{tab:review_estimation}
\end{table}

Point sensors (\textit{e.g.}, inductive-loop detectors, pneumatic tubes, radio-frequency identification (RFID), {\em etc.}) are widely used for traffic density estimation \citep{NI20163}, and they are robust to environment changes (\textit{e.g.,} weather, light) for stable 24/7 estimation. For example, California uses the Caltrans Performance Measurement System (PeMS), consisting of more than 23,000 loop detectors to monitor the traffic state on trunks and arterial roads \citep{PeMS}. However, the expense of deployment and maintenance of point sensors are considerable \citep{loop_cost}. Some advanced techniques such as Vehicular Ad hoc Network (VANET) \citep{Panichpapiboon08} and Unmanned Aerial Vehicle (UAV), can complement point sensors and contribute to traffic density estimation \citep{Zhu2018, Ke2019}. 

Traffic monitoring cameras are an essential part of an urban traffic surveillance system. Cameras are often used for visual inspection of traffic conditions and detection of traffic accidents by traffic engineers sitting in the Traffic Management Centers (TMC). Such cameras are widely distributed in most metropolises, making it possible for large-scale traffic density estimation. Figure~\ref{fig:teaser} provides the composition and snapshots of traffic monitoring cameras in different cities. 
\begin{figure}[h]
    \centering 
    \includegraphics[width=0.95\textwidth]{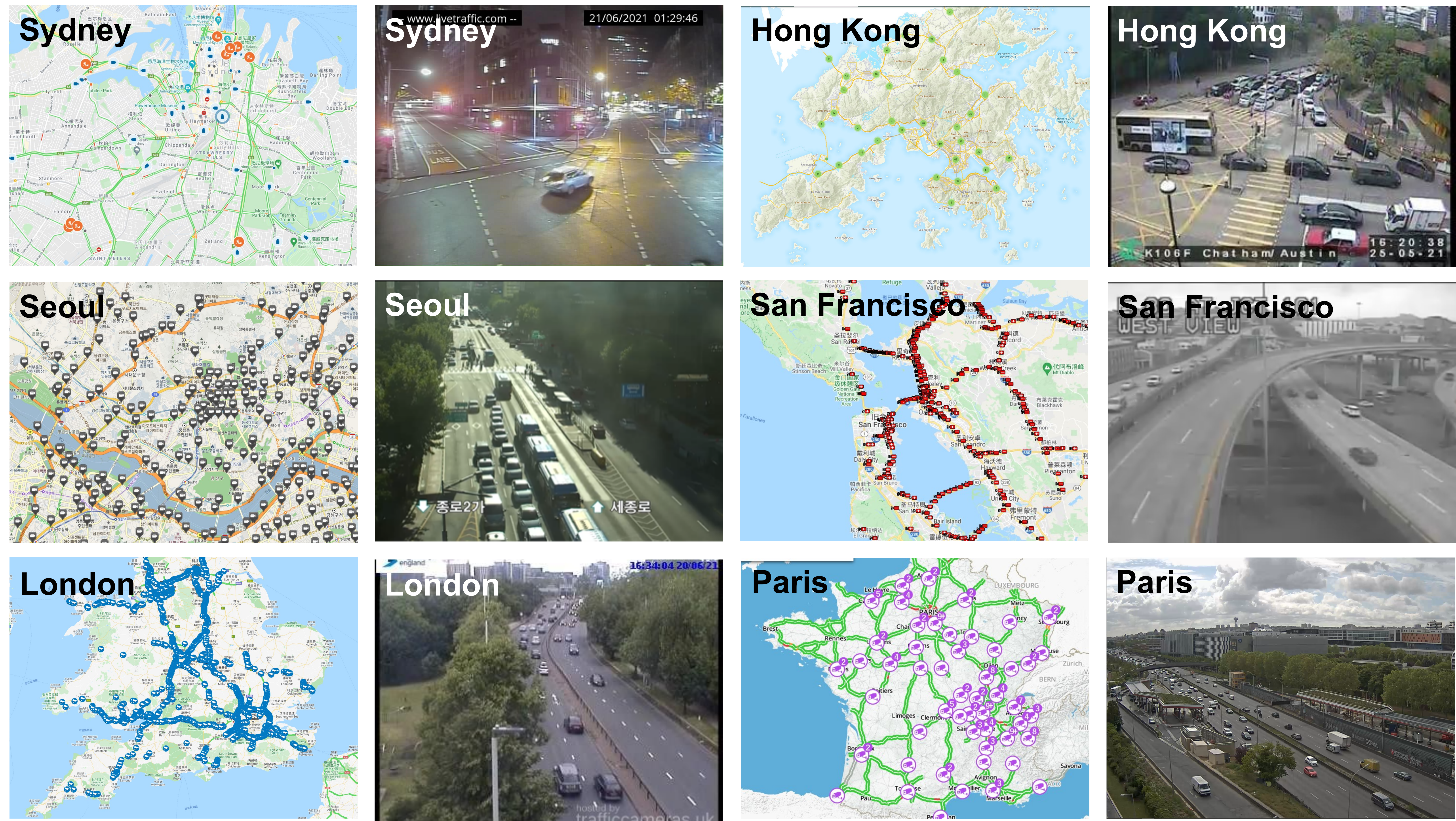}
    \caption{Traffic monitoring cameras in different cities. (Columns \#1 and \#3: the number and locations of cameras in various traffic surveillance systems; Columns \#2 and \#4: snapshots from monitoring cameras.)}
    \label{fig:teaser} 
\end{figure}
For example, in California, approximately 1,300 cameras are set up by Caltrans to monitor the traffic conditions on highways \citep{gerfen2009caltrans}; in Seoul, the TOPIS\footnote{Seoul Transport Operation \& Information Service} system functions on 834 monitoring cameras; and in Hong Kong, the Transport Department uses about 400 monitoring cameras in its eMobility  System\footnote{\url{https://www.hkemobility.gov.hk/en/traffic-information/live/cctv}}. With various camera-based traffic surveillance system deployed globally, there is a great potential to extract traffic information from camera images and videos. Combined with recent advanced technologies, several attempts have been made to vehicle information extraction (\textit{speed and count}) \citep{todd2003camera_speed, wan2014camera,  vehicle_info_extraction}, vehicle re-identification \citep{vrid1, vrid2} and pedestrian detection \citep{pd1, pd2}. Furthermore, it is in great need to make use of the massive traffic monitoring camera data for traffic density estimation.

For different traffic monitoring cameras, the collected data can be low-resolution videos ({\em e.g.}, Seoul TOPIS system) or images ({\em e.g.}, Hong Kong eMobility system). Video data can be used to estimate traffic speed, flow and density, whereas image data can only be used for density estimation, owing to the typically low frame rate of images ({\em e.g.}, one image every two minutes in Hong Kong). To develop a framework that can be applicable to different systems in the world, this paper focuses on the traffic density estimation problem. We also note that the proposed framework can be further extended to make use of the video data for speed and flow estimation. 

To look into the density estimation problem, we note that the traffic density $k$ is computed as the number of vehicles $N$ per lane divided by the length of a road $L$ \citep{DARWISH2015337}, as presented in Equation~\ref{eq:raw_traffic_density}.
\begin{equation}
    k = \frac{N}{L}
    \label{eq:raw_traffic_density}
\end{equation}
Therefore, based on Equation~\ref{eq:raw_traffic_density}, the traffic density estimation problem can be decomposed into two sub-problems: 
\begin{itemize}
    \item {\bf Camera calibration:} aims to estimate the road length $L$ from camera images, in which the core problem is to measure the distance between the real-world coordinates corresponding to the image pixels.
    \item {\bf Vehicle detection:} focuses on counting the vehicle number $N$, and it can be formulated as the object detection problem. 
\end{itemize}
Both problems are separately discussed in the research field of Computer Vision (CV) \citep{Zhang2000checkboard, fast_rcnn}. However, the challenges of traffic density estimation from  monitoring cameras are unique. 

The data collected from traffic monitoring cameras appeals to the {\bf 4L} characteristics. Firstly, due to personal privacy concerns and network bandwidth limits, the camera images are usually in {\bf L}ow resolution and {\bf L}ow frame rate. For example, in Hong Kong, the resolution of the monitoring image is $320 \times 240$ pixels, and all images are updated every two minutes \citep{zhang2010cctv}. Secondly, it is onerous to annotate detailed information for each camera, and hence most of the collected data are {\bf L}acking in annotation. Thirdly, monitoring cameras distributed across urban areas are often {\bf L}ocated in complex road environments, where the roads are not simply straight segments ({\em e.g.}, curved roads, mountain roads and intersections). Overall, we summarize the challenges of the traffic density estimation using the monitoring cameras as {\bf 4L}, which represents: {\bf L}ow resolution, {\bf L}ow frame rate, {\bf L}ack of annotated data and {\bf L}ocated in complex road environments. 

The 4L characteristics present great challenges to both camera calibration and vehicle detection problems. There is a lack of holistic frameworks to comprehensively address the 4L characteristics for traffic density estimation using monitoring cameras.  To further highlight the contributions of this paper, we first review the existing literature on both camera calibration and vehicle detection.


{\bf Literature review on camera calibration.} Camera calibration aims to match invariant patterns (\textit{i.e.,} key points) to acquire a quantitative relationship between the points on images and in the real world. Under the 4L characteristics, conventional camera calibration faces multi-fold challenges: 1) The endogenous camera parameters ({\em e.g.,} focal length) can be different for each camera and are generally unknown. 2) Recognizing the brands and models of vehicles from low-resolution images is challenging, making it difficult to correctly match key points based on car model information; 3) Continuous tracking a single vehicle from low frame rate images is impossible, which makes some of the existing algorithms inapplicable. 4) The invariant patterns in images are challenging to locate. This difficulty is caused by both the locations of the monitoring cameras (usually at the top of buildings or bridges to afford a wide visual perspective for visual monitoring traffic conditions) and the low image resolution.
Even a one-pixel shift of the annotation errors (errors when annotating the key points) will result in a deviation of tens of centimeters in the real world. 5) Existing camera calibration algorithms assume straight road segments, but many monitoring cameras locate at more complex road environments (\textit{e.g.,} curved roads, mountain roads, intersections), making the existing algorithms not applicable. 

Existing camera calibration methods only solve a subset of the aforementioned challenges. In the traditional calibration paradigm, a checkboard with a certain grid length is manually placed under the cameras \citep{Zhang2000checkboard}, and key points can be selected as the intersections of the grid.  However, it is time- and labor-consuming to simultaneously calibrate all cameras in the entire surveillance system. In the transportation scenario, the key points can be extended to the corner of standard objects on the road, such as shipping containers \citep{ke2017camera}, but such objects are not always desired in all surveillance cameras. A common method for traffic camera calibration without the need of specialist equipment is to estimate the camera parameters using the vanishing point method, which leverages the perspective effect. The key points can be selected either as road markings \citep{Cao2004vp, Li2007camera} or common patterns on vehicles on roads \citep{Dubska2014Brno, SOCHOR201787Brno}. These works assume that both sides of the road are parallel straight lines or that all vehicles drive in the same direction. However, this assumption is invalid for complex road environments, such as curved roads and intersections, where vehicles drive in multiple directions. Hence, it is difficult to generalize the method to all camera scenarios in different traffic surveillance systems. 
Another alternative method is the Perspective-n-Point (PnP) method, which does not rely on vanishing points, but estimates the camera orientation given $n$ three-dimensional points and their projected two-dimensional points (Normally $n \geq 3$) in the image. Several algorithms have been proposed to solve the PnP problem \citep{Haralick1991p3p, Quan1999pnp, Lepetit2009epnp, Hesch2011dls, Li2012rpnp}, and it has been validated as a feasible and efficient method of traffic camera calibration using monitoring videos \citep{Bhardwaj2018autocalib}. However, the PnP method requires prior knowledge of the camera focal length, which is unknown for many of the monitoring cameras in real-world applications. The PnP method can be further extended to the PnPf method, which considers the focal length as an endogenous variable during the calibration \citep{Penate2013upnp, Zheng2014gpnp, Wu2015p35p, zheng2016dlspnpf}, but it has rarely been successfully applied to traffic monitoring camera in practice. An important reason is that PnPf is normally sensitive to annotation errors which can lead to a completely false solution. Because the images from traffic monitoring cameras are in low-resolution, the PnPf method may not be applicable. Additionally, a recently reported method \citep{Bartl2020Brno} calibrates the camera in complex road environments without knowing the focal length, but it requires that the key points are on a specific vehicle model {{\em e.g.,} Tesla Model S}, which is impractical for low-resolution and low-frame-rate cameras. In summary, existing camera calibration methods may not be suitable under 4L characteristics. The main reason is that the key points on single vehicle cannot provide enough information for the calibration due to the 4L characteristics. In contrast, if multiple key points on multi-vehicles are considered in the camera calibration method, the calibration results could be made more stable and robust. However, this is still an open problem to the research community.

{\bf Literature review on vehicle detection.} For vehicle detection, current solutions leverage machine-learning-based models to detect vehicles from camera images, while many challenges still remain: 1) The machine learning models heavily rely on the annotated images for supervised training, and the labeled images are generally not available for each traffic surveillance system. 2) Vehicles only occupy several pixels in images due to the low resolution of images, making them difficult to be detected by the machine learning models; 3) during nighttime, the lighting conditions may hinder the detection of vehicles, presenting a challenge to 24/7 traffic density estimation.

Vehicle detection from monitoring cameras has been extensively studied for many years. Background subtraction was initially considered as an efficient algorithm to extract vehicles from the background \citep{ozkurt2009automatic, zhang2010cctv, jain2012traffic}. The underlying assumption in background subtraction is that the background of multiple images is static, and can therefore be obtained by averaging multiple images. However, this assumption may be improper when the illumination intensity of different images varies significantly, such as at night or on windy days. 
Recent studies have focused on detection-based algorithms since they are more resistible to the background changes.
General object detection frameworks can be used to detect vehicles from images \citep{fast_rcnn,redmon2016yolo,mask_rcnn,lin2017focal}, while as they are not tailored for vehicle detection, the performance is not satisfactory. In the transportation community, \citet{Bautista2016cnn} applied a convolutional neural network (CNN) for vehicle detection in low resolution traffic videos; \citet{BISWAS2019176} combined two classical detection frameworks for accuracy consideration; and \citet{Yeshwanth2017OF} extended to automatically segment the region of interest (ROI) based on optical flow. Recently, \citet{zhang2017citycam} generated a weighted mask to compensate for size variance caused by the perspective effect. They subsequently combined a CNN with Long-Short-Term Memory (LSTM) to exploit spatial and temporal infromation from videos \citep{zhang2017fcn}. In summary, the performance of existing vehicle detection models is degraded drastically when annotated data are lacking. In particular, the unified performances under different camera lighting conditions cannot be guaranteed. It is potentially possible to adopt transfer learning to fuse multiple data from different traffic scenarios, while the related study is still lacking in transportation community.

Overall, the challenges to traffic density estimation under 4L characteristics are summarized in Figure~\ref{fig:literature_review}.
\begin{figure}[th]
    \centering
    \includegraphics[width=1\textwidth]{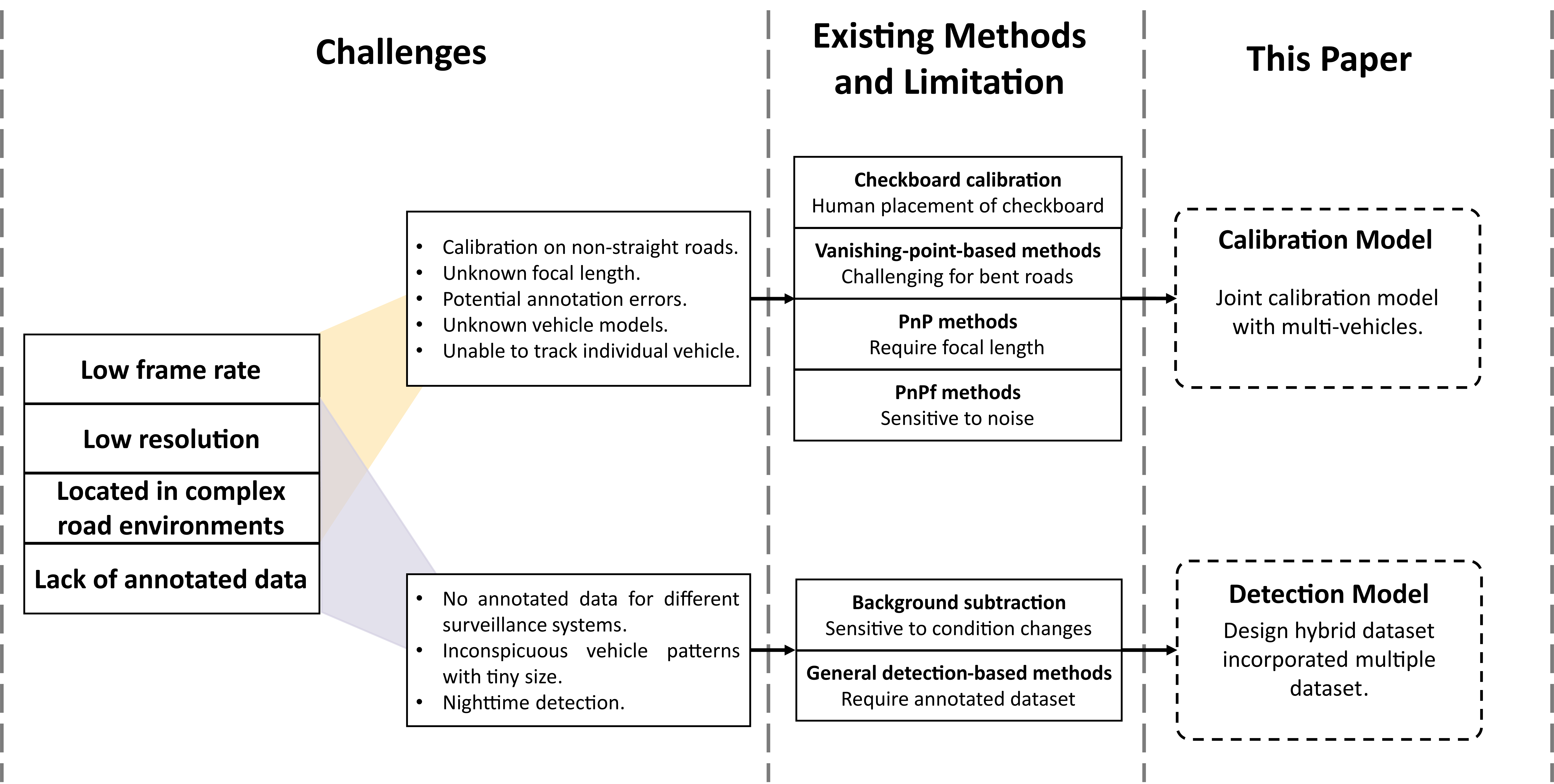}
    \caption{Framework of traffic density estimation with monitoring cameras.}
    \label{fig:literature_review}
\end{figure}
For road length estimation, we aim to calibrate the monitoring camera with unknown focal length using low-quality image slice obtained under complex conditions. For vehicle number estimation, we focus on developing a training strategy that is robust for low-resolution images acquired in both daytime and at nighttime without annotating extra images. 

This paper proposes a holistic framework that turns traffic monitoring cameras into intelligent sensors for traffic density estimation. The proposed framework mainly consists of two component: 1) camera calibration and 2) vehicle detection. For camera calibration, a novel method of multi-vehicle camera calibration (denoted as  \texttt{MVCalib}) is developed to utilize the key point information of multiple vehicles simultaneously. The actual road length can be estimated from the pixel distance in images once the camera is calibrated.  For vehicle detection, we develop a linear-program-based approach to hybridize various public vehicle dataset to balance the images in daytime and at nighttime under various conditions, and these public datasets are originally used for different purposes. A deep learning network is trained on the hybrid dataset, and it can be used for different vehicle detection task in various monitoring camera system. Two case studies with ground truth have been conducted to evaluate the performance of the proposed framework. Results show that the estimation accuracy for the road length is more than 95\%. The vehicle detection can reach an accuracy of 88\% in both daytime and at nighttime, under low-quality camera images. It is demonstrated that this framework can be applied to real time traffic density estimation using monitoring cameras in different countries.

To summarize, the major contributions of this paper are as follows:
\begin{itemize}   
    \item It provides a holistic framework for 24/7 traffic density estimation using traffic monitoring cameras with {\bf 4L} characteristics: {\bf L}ow frame rate, {\bf L}ow resolution, {\bf L}ack of annotated data, and {\bf L}ocated in complex road environments.
    \item It first time develops a robust multi-vehicle camera calibration method  \texttt{MVCalib} that collectively utilizes the spatial relationships among key points from multiple vehicles.
    \item It systematically designs a linear-program-based data mixing strategy to synergize image datasets from different cameras and to enhance the performance of the deep-learning-based vehicle detection models. 
    \item It validates the proposed framework in two traffic monitoring camera systems in Hong Kong and California, and the research outcomes create portals for rapid and massive deployment of the proposed framework in different cities.
\end{itemize}

The rest of this paper is organized as follows. Section \ref{section:methods} presents the proposed methods for camera calibration and vehicle detection separately. Section \ref{section:evaluation} focuses on experiments and evaluations of the current framework. In section \ref{section:case_study_1}, a case study at the footbridge of the Hong Kong Polytechnic University (PolyU) is presented and a case study with the Caltrans system follows in section \ref{section:case_study_2}. Finally, conclusions are drawn in section \ref{section:conclusion}. 

\section{Methods}
\label{section:methods}
In this section, we first introduce the overall framework, and the camera calibration model and vehicle detection model are then elaborated separately. All notation used in this paper is summarized in \ref{apx:notations}.

\subsection{The Overall Framework}
The framework of the traffic density estimation model is shown in Figure~\ref{fig:framework}.
\begin{figure}[h]
    \centering
    \includegraphics[width=1\textwidth]{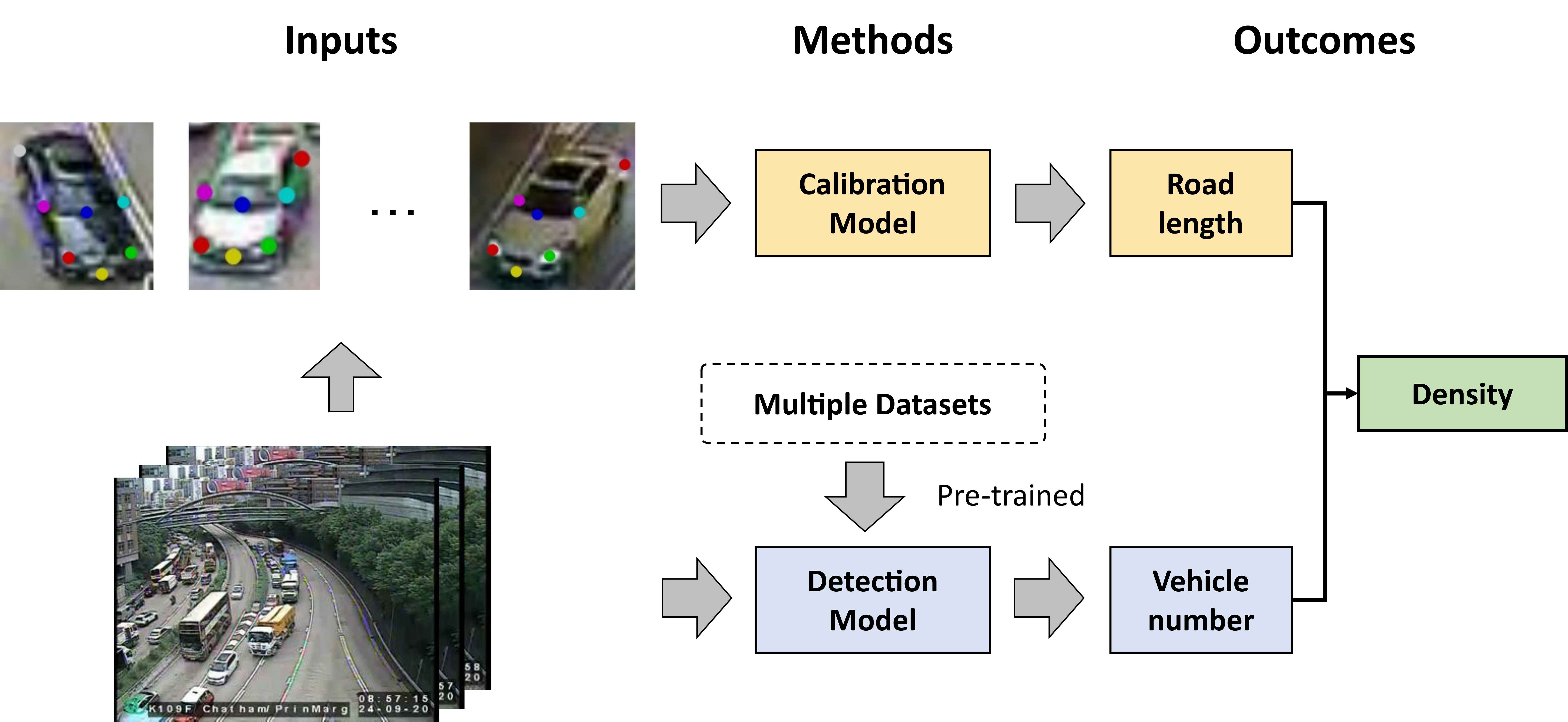}
    \caption{Framework of traffic density estimation with monitoring cameras.}
    \label{fig:framework}
\end{figure}
Camera images are first collected from public traffic monitoring camera systems, and then key points on vehicles are annotated. The camera calibration model uses the annotated data to  derive a relationship between points on images and in the real world. If we can acquire the skeleton of the road, the road length can be further computed after calibration. For vehicle detection, the camera image data are fed to a deep-learning-based vehicle detection model pre-trained on a hybridized dataset, which is used to count the vehicle on the road. Combining the road length and vehicle number information, we can estimate the high-granular traffic density information on the road. We note that the proposed framework can be run in real-time, and hence the output of the framework has great potential in supporting real-time traffic operations and management applications.

\subsection{Camera Calibration}
In this section, we present the proposed camera calibration method, \texttt{MVCalib}. The background about camera calibration is first reviewed, then the detailed information about the proposed camera calibration model will be elaborated subsequently.

\subsubsection{Overview of camera calibration problems}
A simplified pinhole camera model is widely used to illustrate the relationship between three-dimensional objects in the real world and the projected two-dimensional points on the camera images. Given the location of a certain point in the real world $\left[ X, Y, Z \right]^T \in \mathbb{R}^3$, the projected point on the camera image can be represented as $\left[ u, v \right]^T \in \mathbb{R}^2$. The relationship between $\left[ X, Y, Z \right]^T$ and $\left[ u, v \right]^T$ is defined in Equation \ref{eq:full_PnP_eq} and \ref{eq:abbr_PnP_eq}.

\begin{equation}
    s \begin{bmatrix} u \\ v \\  1 \end{bmatrix} = 
    \begin{bmatrix} f & 0 & \frac{w}{2} \\ 0 & f & \frac{h}{2} \\ 0 & 0 & 1 \end{bmatrix}
    \begin{bmatrix} 
        r_{11} & r_{12} & r_{13} & t_1\\ 
        r_{21} & r_{22} & r_{23} & t_2\\ 
        r_{31} & r_{32} & r_{33} & t_3\\ 
    \end{bmatrix}
    \begin{bmatrix} X \\ Y \\ Z \\ 1 \end{bmatrix}
    \label{eq:full_PnP_eq}
\end{equation}
\begin{equation}
    sp_i = K\left[ R \vert T \right]P_i
    \label{eq:abbr_PnP_eq}
\end{equation}
where Equation \ref{eq:abbr_PnP_eq} is the vectorized version of Equation \ref{eq:full_PnP_eq}. $K =  \begin{bmatrix} f & 0 & \frac{w}{2} \\ 0 & f & \frac{h}{2} \\ 0 & 0 & 1 \end{bmatrix} $ encodes the endogenous camera parameters, where $f$ denotes the focal length of the camera. $w$ and $h$ represent the width and height of images. $R =  \begin{bmatrix} 
        r_{11} & r_{12} & r_{13}\\ 
        r_{21} & r_{22} & r_{23}\\ 
        r_{31} & r_{32} & r_{33}\\ 
    \end{bmatrix}$ and $T = \begin{bmatrix} 
        t_1\\ 
        t_2\\ 
        t_3\\ 
    \end{bmatrix}$ are the rotation matrix and translation vector of the camera, respectively. Hence, $[f, R, T] \in \mathbb{R}^{13}$ are the $13$ parameters to be estimated in the problem of camera calibration. Once the camera parameters $f, R$, and $T$ are calibrated, the location of projection points on an image can be deduced from the coordinates in the real-world system

The key points on vehicles in two-dimensional images and the three-dimensional real world are typically common features such as headlights, taillights, license plates, {\em etc.} Existing camera calibration methods assume that the key points of a specific vehicle model ({\em e.g.,} Tesla Model S, Toyota Corolla) are known. Under the 4L characteristics, camera images are too blurry for us to distinguish vehicle models. Hence, in the proposed method, a set of model candidates is built to serve as the references of three-dimensional points. The dataset of two-dimensional and three-dimensional key points for the $i$th vehicle in images and in the real world can be represented as

\begin{equation}
    \begin{aligned}
        \mathcal{p}^{i} &= \left \{ p_1^{i}, p_2^{i}, \dots,  p_k^i, \dots, p_{M_{i}}^{i} \right \}, \quad i \leq n \\
    \mathcal{P}^{i}_j &= \left \{ P_{j, 1}^i, P_{j, 2}^i, \cdots, P_{j, k}^i, \cdots,  P_{j, M_i}^i  \right \}, \quad i \leq n, j \leq m
    \end{aligned}
    \label{eq:vehicle_index_def}
\end{equation}

\noindent where $i$ represents the vehicle index in images and $j$ represents the index of vehicle models. $\mathcal{p}^{i}$ represents the set of two-dimensional key points of the $i$th vehicle on camera images, and $\mathcal{P}_j^{i}$ denotes the sets of three-dimensional key points of the $i$th vehicle in real world assuming the vehicle model is $j$. $n$ and $m$ represent the number of vehicles and the number of vehicle models in the real world, respectively. $M_i$ denotes the number of key points on the $i$th vehicle. More specifically, $p_k^i$ represents the location of the $k$th key point on vehicle $i$ in the image, and $P_{j, k}^i$ represents the three-dimensional coordinates of the $k$th key point on vehicle $i$ assuming that the vehicle model is $j$.

\subsubsection{$\texttt{MVCalib}$}
In this section, we present the proposed multi-vehicle camera calibration method \texttt{MVCalib}. The pipeline of $\texttt{MVCalib}$ is shown in Figure~\ref{fig:framework_calib}. 
\begin{figure}[h]
    \includegraphics[width=1\textwidth]{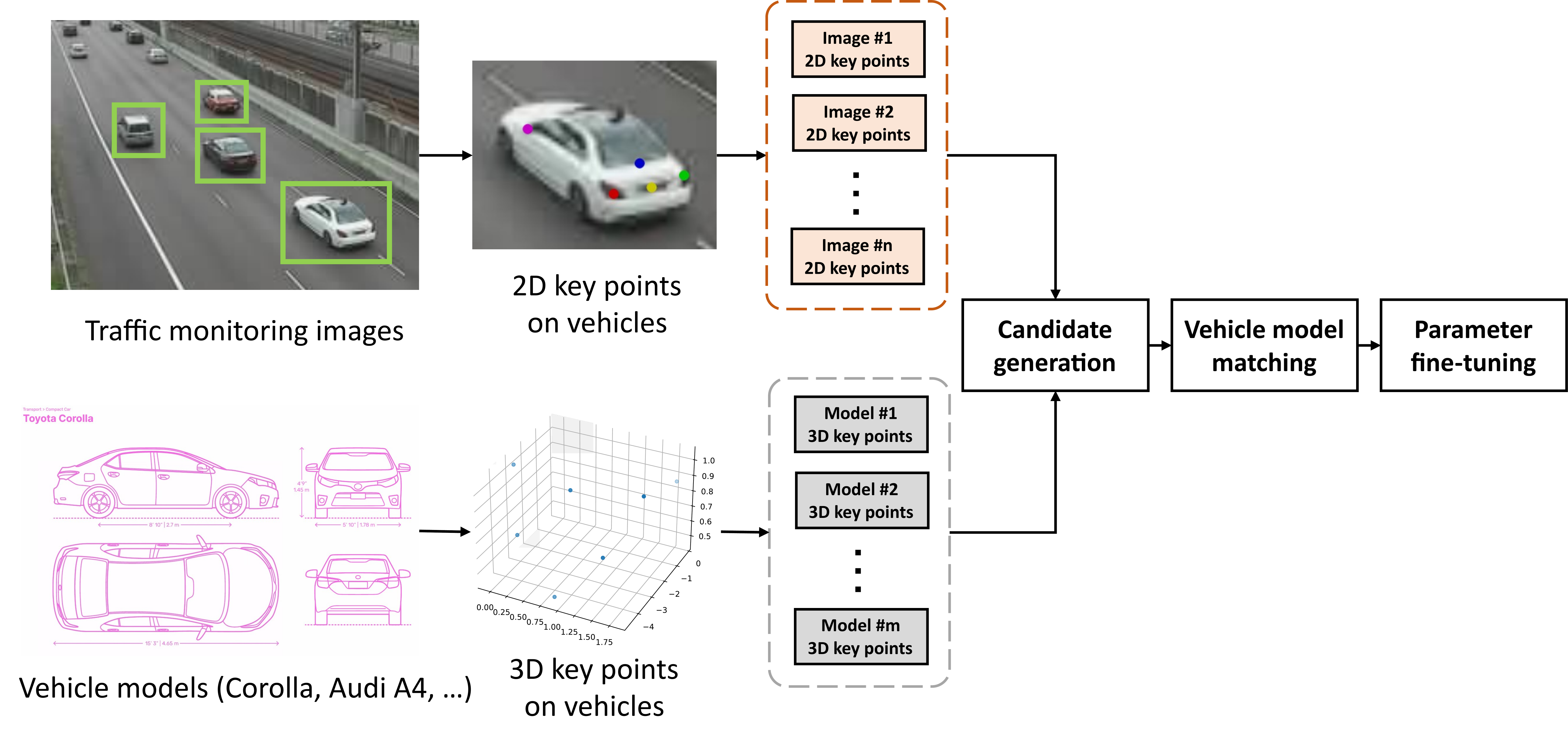}
    \caption{The pipeline of the \texttt{MVCalib} method for camera calibration.}
    \label{fig:framework_calib}
\end{figure}

$\texttt{MVCalib}$ proceeds through three stages: candidate generation, vehicle model matching and parameter fine-tuning. In the candidate generation stage, the solution candidates for each vehicle are generated separately based on conventional camera calibration methods. In the vehicle model matching stage, a specific model is assigned to each vehicle in the camera images. In the parameter fine-tuning stage, joint information on multiple vehicles is utilized to fine-tune the camera parameters. The fine-tuned value of $f, R, T$ will be carried out to estimate the road length for the traffic density estimation.

{\bf Candidate generation.}
In the candidate generation stage, we first apply the conventional camera calibration method to the key points on each vehicle, assuming that its vehicle model and the focal length of the camera are known. Mathematically, for the $i$th vehicle, the coordinates of $M_i$ pairs of key points in two-dimensional space $\mathcal{p}^i$ and in three-dimensional space $\mathcal{P}_{j}^i$ under the $j$th model are known. Given a default value of focal length $\hat{f}$, the parameters of rotation matrix $R$ and translation vector $T$ can be estimated through the Efficient PnP algorithm (EPnP) \citep{Lepetit2009epnp} with a random sample consensus (RANSAC) strategy \citep{Fischler1981RANSAC}.

The EPnP method is applied to all pairs of $(i,j)$, and hence a total number of $m \times n$ times of estimation using EPnP are conducted. The estimated camera parameters (candidates) is denoted as $\tilde{\psi}_j^i = \left\{ \hat{f}, \tilde{R}_j^i, \tilde{T}_j^i \right\}$, which represents the focal length, rotation matrix and translation vector for the $i$th vehicle of the $j$th model.

{\bf Vehicle model matching.}
In the vehicle model matching stage, the most closely matched vehicle model is determined to minimize the projection error from the real world to the image plane for each vehicle $i$. Mathematically, we aim to select the best vehicle model $j$ from $\tilde{\psi}_j^i$ to obtain the camera parameter $\psi^i$ for each vehicle. In the candidate generation stage, the focal length is fixed to a default value, which may contribute to errors in the projection. Therefore, in this stage, we adjust the focal length to a more accurate value and refine the parameter estimation. To this end, we formulate an optimization problem with the objective of minimizing the projection loss from three-dimensional real world to two-dimensional camera images, as presented in Equation~\ref{eq:2d_loss}.

\begin{equation}
\label{eq:2d_loss}
\begin{array}{rrcllll}
\vspace{5pt}
 L_v\left( \psi_j^i \right) = \displaystyle \min_{\psi_j^i} & \multicolumn{4}{l}{\displaystyle  \sum_{k=1}^{M_i} \left\Vert p_k^i - \frac{1}{s_{j, k}^i} K\left( f_j^i \right) \left[ R_j^i \vert T_j^i \right] P_{j, k}^i \right\Vert_2 } \\
\textrm{s.t.} & \psi_j^i &= & \left\{ f_j^i, R_j^i, T_j^i \right\} & \\
~ & s_{j, k}^i&= &R_j^i \Big\vert_3 \cdot P_{j, k}^i + T_j^i \Big\vert_3 \\
~ & f_j^i & \geq & 0, \forall  i \leq i, j \leq m&
\end{array}
\end{equation}

\noindent where $L_v \left(\cdot \right)$ defines the projection loss from the three-dimensional real world to two-dimensional images for the key points on vehicles. $s_{j, k}^i = R_j^i \Big\vert_3 \cdot P_{j, k}^i + T_j^i \Big\vert_3$ is the scale factor for the combination of the $k$th key point on the $i$th vehicle with the $j$th model. $R_j^i \Big\vert_3$ represents the third row of the rotation matrix and
$T_j^i \Big\vert_3$ denotes the third element of the translation vector. The focal length of a camera $f_j^i$ should be greater than $0$.

To solve the optimization problem $L_v(\psi_j^i)$, we employ the Covariance Matrix Adaptation Evolution Strategy (CMA-ES) \citep{Hansen1996CMAES}, which is an evolutionary algorithm for non-linear and non-convex optimization problems, to search for the optimal parameter $\psi_j^i$ for each combination of vehicle and vehicle model. As the performance of the CMA-ES depends on the initial points, we start by searching for the parameters from $\tilde{\psi}_j^i$. For vehicle $i$, we assign the vehicle model with the minimal projection loss $L_v$, as presented in Equation~\ref{eq:proj}.

\begin{equation}
\label{eq:proj}
\psi^i = \left\{ f^i, R^i, T^i \right\} = \argmin L_v \left( \psi_j^i \right), \quad \forall i \leq n
\end{equation}

{\bf Parameter fine-tuning.}
In this stage, we combine the key point information on multiple vehicles and further fine-tune the information to obtain the final estimation of the camera parameters $\psi$. In previous stages, we made use of the key point information on each single vehicle, and applied the estimated camera parameter $\psi^i$ to each vehicle $i$ separately. Ideally, if $\psi^i$ is perfectly estimated, we can project the key points on all vehicles in camera images back to the real world using $\psi^i$, and those key points should exactly match the key points on the vehicle models. Based on this criterion, we can select the camera parameters from $\psi^i$ and further fine-tune them to obtain $\psi$.

To this end, we back-project the two-dimensional points in camera images to the three-dimensional real world by using the parameter $\psi^{i'}$ for vehicle $i'$ as an ``anchor''. Mathematically, given an $i$th vehicle, 
the coordinates of the $k$th key point on the camera image and in the real world can be represented as $p^i_k$ and $P^i_{k}$, respectively. Note that $P^i_{k}$ is a member of $P^i_{j,k}$ as  the vehicle model is fixed in the vehicle model matching stage. To back-project $p^i_k$ to the real-world space using $\psi^{i'}$, we solve a system of equations derived from Equation \ref{eq:full_PnP_eq}, as shown in Equation~\ref{eq:backproj}.

\begin{equation}
\label{eq:backproj}
    \left\{
        \begin{array}{lr}
            \left(
                \tilde{u}_k^i
                \left[ R^{i'}\vert T^{i'} \right] \Big\vert_3 -
                \left[ R^{i'}\vert T^{i'} \right] \Big\vert_1
            \right) \cdot 
            \begin{bmatrix}  \hat{P}_{k}^i(\psi^{i'}) \\ 1 \end{bmatrix} = 0 \\
            \left(
                \tilde{v}_k^i 
                \left[ R^{i'}\vert T^{i'} \right] \Big\vert_3 -
                \left[ R^{i'}\vert T^{i'} \right] \Big\vert_2
            \right) \cdot 
            \begin{bmatrix}  \hat{P}_{k}^i(\psi^{i'}) \\ 1 \end{bmatrix} = 0
        \end{array}
    \right.
\end{equation}

\noindent where $\tilde{u}_k^i = \frac{u_k^i - \frac{w}{2}}{f^{i'}}, \tilde{v}_k^i = \frac{v_k^i - \frac{h}{2}}{f^{i'}}$,  $(u_k^i, v_k^i)$ is the two-dimensional coordinate of the $k$th key point on vehicle $i$ in the camera images, and $\hat{P}_{k}^i(\psi^{i'})$ represents the back-projected point on the $i$th vehicle of the $k$th key point given the camera parameter $\psi^{i'}$ of anchor vehicle $i'$. 

The primary loss between back-projected points and real-world points is defined in Equation~\ref{eq:3d_loss}.

\begin{equation}
\label{eq:3d_loss}
\begin{array}{lllll}
 L_p \left( i, \psi^{i'} \vert \alpha \right) &=& \sum_{k_1 < k_2 < M_i}
    \xi_l \left( i, k_1, k_2, \psi^{i'} \right) + \alpha \xi_r \left( i, k_1, k_2, \psi^{i'} \right) \\
    \xi_l \left( i, k_1, k_2, \psi^{i'} \right) &=& \left| \left\Vert \overrightarrow{\hat{P}_{k_1}^i(\psi^{i'}) \hat{P}_{k_2}^i(\psi^{i'})} \right\Vert_2 - 
    \left\Vert \overrightarrow{P_{k_1}^i P_{k_2}^i} \right\Vert_2 \right| \\
    \xi_r \left( i, k_1, k_2, \psi^{i'} \right) &=&
    \sqrt{1 - \left(
        \frac{
            \overrightarrow{P_{k_1}^i P_{k_2}^i} \cdot
            \overrightarrow{\hat{P}_{k_1}^i(\psi^{i'}) \hat{P}_{k_2}^i(\psi^{i'})}
    }
            {
                \left\Vert \overrightarrow{P_{k_1}^i P_{k_2}^i} \right\Vert_2 \cdot
                \left\Vert \overrightarrow{\hat{P}_{k_1}^i(\psi^{i'}) \hat{P}_{k_2}^i(\psi^{i'})} \right\Vert_2}
    \right)^2}
\end{array}
\end{equation}

\noindent where $\xi_l \left( i, k_1, k_2, \psi^{i'} \right)$ and $\xi_r \left( i, k_1, k_2, \psi^{i'} \right)$ represent the distance and angle loss between the back-projected points and real-world points, and $\alpha$ is a hyper-parameter that adjusts the weight of each loss. $\left\Vert \overrightarrow{P_{k_1}^i P_{k_2}^i} \right\Vert_2$ and $\left\Vert \overrightarrow{\hat{P}_{k_1}^i(\psi^{i'}) \hat{P}_{k_2}^i(\psi^{i'})} \right\Vert_2$ are vectors that consists of any two real-world and back-projected points on the same vehicle $i$. $k_1$ and $k_2$ represents two non-overlapping indices of the key points on the same vehicle $i$. The distance loss represents the gap between the Euclidean distance of the back-projected points and the one of the real-world points, while the angle loss can be regarded as the sine value of the angle between two vectors formed with the back-projected points and real-world points. 
We further aggregate the loss $L_p \left( i, \psi^{i'} \vert \alpha \right)$ for different vehicles $i$ based on their relative distance. In general, if a vehicle is further from the anchor vehicle, then the loss in the back-projected points is larger, and we have less confidence in these points. Therefore, smaller weights are assigned to vehicles that are fruther from the anchor vehicle. 

The objective of minimizing the fine-tuning loss for all vehicles is formulated to consider different weights due to the relative distance, as presented in Equation~\ref{eq:obj_func}.

\begin{equation}
\label{eq:obj_func}
\begin{array}{cllllll}
    L_f \left( \psi^{i'} \vert \alpha, \tau \right) &=&
\sum_{i < n} \omega \left( \hat{C}_i, \hat{C}_{i'} \vert \tau \right)
L_p \left(  i, \psi^{i'} \vert \alpha \right) \\
\omega \left( \hat{C}_i, \hat{C}_{i'} \vert \tau \right) &=&
\frac{\exp \left( \tau \left\Vert \overrightarrow{\hat{C}_i \hat{C}_{i'}} \right\Vert_2 \right)}
{\sum_{{i''} < n}
\exp \left( \tau \left\Vert \overrightarrow{\hat{C}_{i'} \hat{C}_{i''}} \right\Vert_2 \right)} \\
\hat{C}_i &=& \frac{1}{M_i} \sum_{k < M_i} \hat{P}_{k}^i(\psi^{i'})\\
\hat{C}_{i'} &=& \frac{1}{M_{i'}} \sum_{k < M_{i'}} \hat{P}_{k}^{i'}(\psi^{i'})
\end{array}
\end{equation}

\noindent where $\hat{C}_i$ is the centroid of all back-projected key points on the $i$th vehicle, and $\hat{C}_{i'}$ is the centroid of all back-projected key points on the anchor vehicle $i'$.
$\omega \left( \hat{C}_i, \hat{C}_{i'} \vert \tau \right)$ is the weighting function for vehicle $i$ using the vehicle $i'$ as an anchor. The temperature $\tau$ is a hyper-parameter that controls the distribution of the weighting function. When $\tau = 0$, the weighting function uniformly averages the loss for all vehicles; when $\tau < 0$, more attention will be paid to vehicles that are close to the current vehicle, and vice versa.

To obtain the final estimation of the camera parameters, we minimize the objective $L_f \left( \psi^{i'} \vert \alpha, \tau \right)$ in Equation~\ref{eq:obj_func} for each selection of anchor vehicle. The optimal estimation is selected as that with the minimal loss, as shown in Equation~\ref{eq:final}.

\begin{equation}
\label{eq:final}
\psi = \argmin_{i' < n} L_f \left( \psi^{i'}\vert \alpha, \tau \right)
\end{equation}

As the optimization problem presented in Equation~\ref{eq:final} is non-linear and non-convex, the CMA-ES is again leveraged to solve the optimization problem using $\psi^i$ as initial values.
We note that the parameter space of Equation~\ref{eq:final} is $\psi \in \mathbb{R}^{n \times 13}$, which includes $f, R, T$. $f$ denotes a scalar of the focal length, while $R$ and $T$ are a $3 \times 3$ rotation matrix and a $3 \times 1$ translation vector, respectively. In total, there are $13$ unknown parameters in $\psi$ to be estimated. To further simplify the parameter space in $\psi$,  Proposition~\ref{prop:reduce} is proven to reduce the dimension of parameters to $\mathbb{R}^{7}$.

\begin{proposition}
\label{prop:reduce}
The rotation matrix $R$ can be encoded as a scalar of angle $\theta$ and a vector of rotation axis $\bm{d} \in \mathbb{R}^{3}$.
\end{proposition}
\begin{proof}
See \ref{ap:reduce}.
\end{proof}

\subsection{Vehicle Detection}
\label{section:detection}
In this section, we present the vehicle detection model, which counts the number of vehicles on road segments from camera images. The state-of-the-art vehicle detection models adopt Deep Learning (DL) based methods to train the model on a vehicle-related dataset. The training process of DL models usually requires massive data. Owing to the 4L characteristics, the quantity of annotated camera images for a specific traffic surveillance system cannot support the complete training of a modern DL-based vehicle detection model. In addition, it is inefficient to train new models for each traffic surveillance system separately. Therefore, we adopt the transfer learning scheme to first train the model on traffic-related public datasets, and then apply the model to specific monitoring camera systems \citep{transfer_learning}. 

Existing public datasets are designed for a range of purposes, such as vehicle re-identification (reID), autonomous driving, vehicle detection, {\em etc.}  \citep{lin2014microsoft,deng2009imagenet,yu2018bdd100k,zhang2017citycam,lyu2018ua,dong2015vehicle,luo2018mio}. The camera images in different datasets have different endogenous attributes ({\em e.g.}, focal length, type of photosensitive element, resolution, {\em etc.}) and exogenous attributes ({\em e.g.}, perspective, illumination, directions, {\em etc.}) Additionally, the datasets differ in size. A summary of the existing traffic-related public datasets is presented in Table~\ref{tab:datasets_all}, and snapshots of some of the datasets are shown in Figure~\ref{fig:detection_dataset_preview}. 

\begin{table}[h]
    \centering
    \setlength{\abovecaptionskip}{0pt}    
    \setlength{\belowcaptionskip}{5pt}
    \caption {A summary of traffic-related image datasets.} 
    \label{tab:datasets_all}
    \begin{tabular}{p{0.15\columnwidth}|p{0.08\columnwidth}p{0.11\columnwidth}p{0.17\columnwidth}p{0.35\columnwidth}}
        \hline 
        \textbf{Name}  & \textbf{Size} & \textbf{Resolution} & \textbf{Camera Angle} & \textbf{Original Usage}\\ 
        \hline
        BDD100K & 100,000 & $1280 \times 720$ & Front & Autonomous driving \\
        BIT Vehicle & 9,850 & Multiple & Inclined top & Vehicle reID\\
        CityCam & 60,000 & $352 \times 240$ & Inclined top & Vehicle detection \\
        COCO & 17,684 & Multiple &  Multiple & Object detection \& segmentation \\
        MIO-TCD-L & 137,743 & $720 \times 480$ &  Inclined top & Vehicle detection \& classification \\
        UA-DETRAC & 138,252 & $960 \times 540$ & Inclined top & Vehicle detection \\
        \hline
    \end{tabular} 
\end{table}

\begin{figure}[h]
    \centering
    \includegraphics[width=0.94\columnwidth]{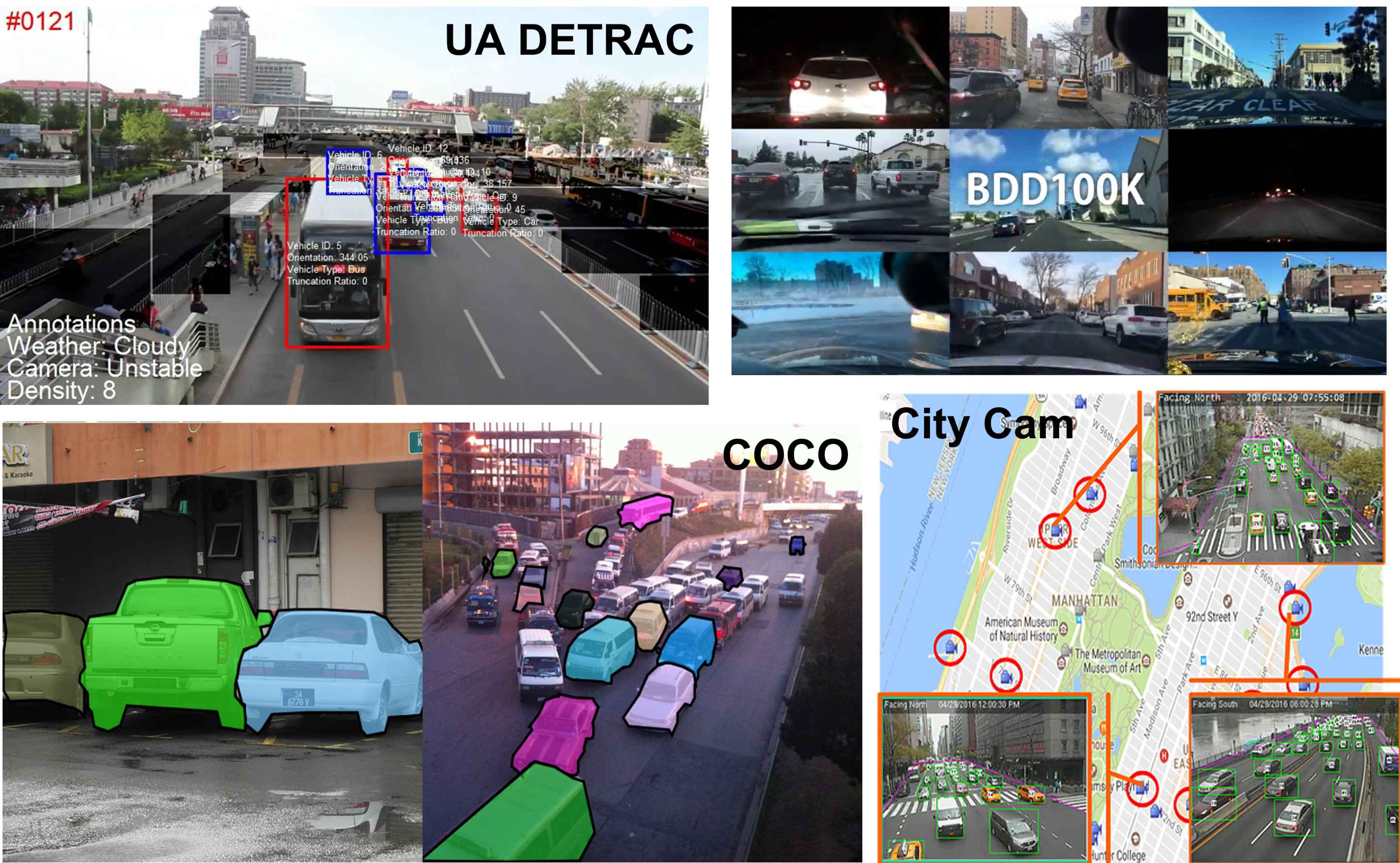}
    \caption{A glance of various traffic-related image datasets used in this study.}
    \label{fig:detection_dataset_preview}
\end{figure}

We categorize the camera images from these datasets into different traffic scenarios, which include time of day (daytime and nighttime), congestion level, surrounding environment, {\em etc.} Each traffic scenario represents a unique set of features in the camera images, so if a DL model is trained for one traffic scenario, it might not perform well on a different scenario. Given the 4L characteristics, the camera images in a large-scale traffic surveillance system may cover multiple traffic scenarios, so it is important to merge and balance the different datasets summarized in Table~\ref{tab:datasets_all} for training the vehicle detection model.

To this end, we formulate a linear program (LP) to hybridize a generalized dataset called the \textbf{LP hybrid dataset}, by re-sampling from multiple traffic-related public datasets. The LP hybrid dataset balances the proportion of images from each traffic scenario to prevent one traffic scenario dominating the dataset. For example, if most camera images are captured during daytime, then the trained vehicle detection model will not perform well on the nighttime images. If different traffic scenarios are comprehensively covered, balanced, and trained, the robustness and generalizability of the detection model will be significantly improved. 

Following the above discussion, the pipeline for the vehicle detection model is presented in Figure~\ref{fig:flowchart_detection}.
\begin{figure}[h]
    \centering
    \includegraphics[width=0.94\columnwidth]{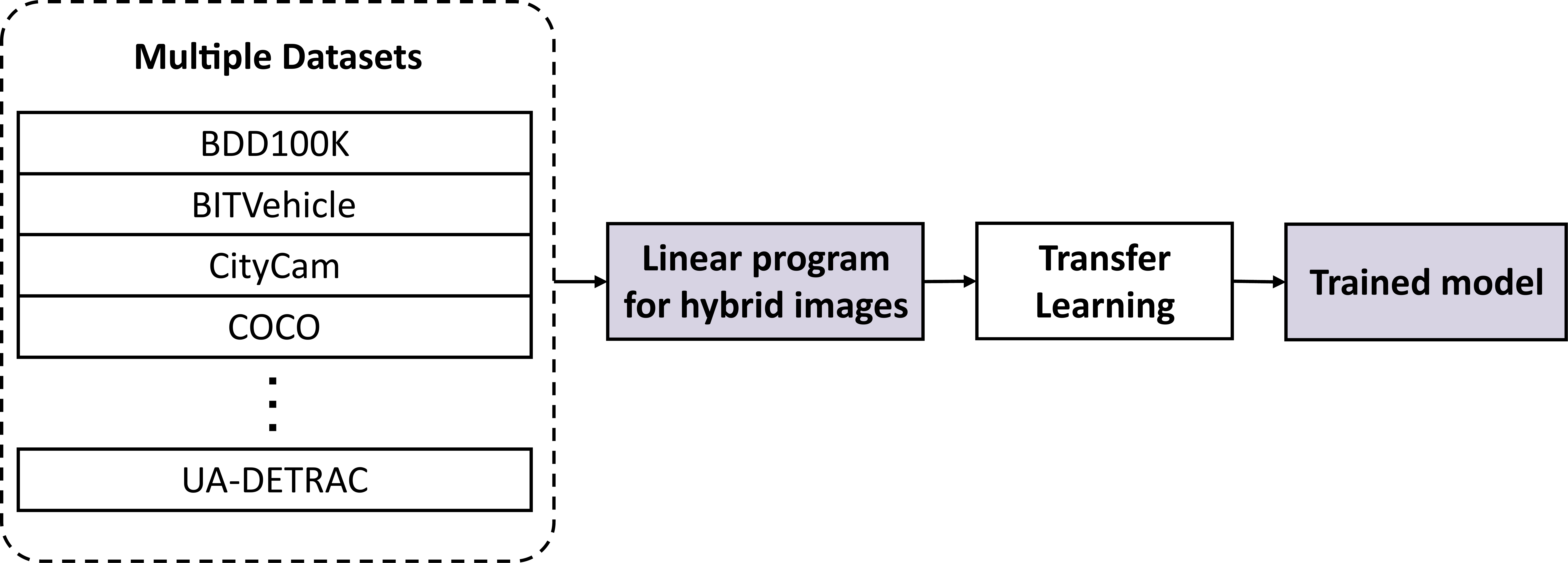}
    \caption{The pipeline of vehicle detection.}
    \label{fig:flowchart_detection}
\end{figure}
One can see that the multiple traffic-related datasets are feed into the LP to generate the LP hybrid dataset, and the dataset will be used to train the vehicle detection model. The trained model can be directly applied to different traffic surveillance systems. 

As stated above, the hybrid detection dataset is formulated as a LP, the goal of which is to maximize the total number of images in the dataset, written as
\begin{equation}
    \max \sum_{\mu=0}^{u-1} \sum_{\nu = 0}^{v-1} q_{\mu,\nu}
    \label{eq:vehicle_LP}
\end{equation}
\noindent where $u$ denotes the number of datasets, and $v$ represents the number of traffic scenarios. $q_{\mu, \nu}$ are decision variables that denotes the number of images to be incorporated into the LP hybrid dataset from dataset $\mu$ for traffic scenario $\nu$.

The constraints of the proposed LP are constructed based on  two principles: 1) The difference between the numbers of images from different traffic scenarios should be limited within a certain range. 2) The number of images contributed by each dataset should be similar. Mathematically, the constraints are presented in Equation~\ref{eq:vehicle_LP_cons}.
\begin{equation}
    \begin{array}{llllllll}
        q_{\mu, \nu} - \frac{\left( 1 + \beta \right) \sum_{\mu=0}^{u-1} q_{\mu, \nu}}{\sum_{\mu=0}^{u-1} \delta(0, Q_{\mu, \nu})} &\leq& 0, &\forall 0 \leq \mu < u, 0 \leq \nu < v, Q_{\mu, \nu} \neq 0 \\
        q_{\mu, \nu} - \frac{\left( 1 - \beta \right) \sum_{\mu=0}^{u-1} q_{\mu, \nu}}{\sum_{\mu=0}^{u-1} \delta(0, Q_{\mu, \nu})} &\geq& 0, &\forall 0 \leq \mu < u, 0 \leq \nu < v, Q_{\mu, \nu} \neq 0 \\
        \sum_{\mu=0}^{u-1} q_{\mu, \nu} - \frac{1 + \gamma}{v} \sum_{\mu=0}^{u-1} \sum_{\nu = 0}^{v-1} q_{\mu, \nu} &\leq& 0, &\forall 0 \leq \nu < v \\
        \sum_{\mu=0}^{u-1} q_{\mu, \nu} - \frac{1 - \gamma}{v} \sum_{\mu=0}^{u-1} \sum_{\nu = 0}^{v-1} q_{\mu, \nu} &\geq& 0, &\forall 0 \leq \nu < v \\
        q_{\mu, \nu} &\leq& Q_{\mu, \nu}, &\forall 0 \leq \mu < u, 0 \leq \nu < v\\
        q_{\mu, \nu} &\geq& 0, &\forall 0 \leq \mu < u, 0 \leq \nu < v
    \end{array}
    \label{eq:vehicle_LP_cons}
\end{equation}

\noindent where the former two constraints adjust the image contribution from different datasets, while the latter two balances the number of images from different traffic scenarios. $Q_{\mu, \nu}$ represents the total number of data for traffic scenario $\nu$ in dataset $\mu$, and $q_{\mu, \nu} \leq Q_{\mu, \nu}$ enforces that the selected number of images should be smaller than the total number of images. $\beta$ is the maximum tolerance parameter for the upper and lower bound of the image number in different traffic datasets given certain scenarios, and $\gamma$ is another maximum tolerance parameter limiting the difference between the numbers of images selected from different scenarios. $\delta(0, Q_{\mu, \nu})$ is defined as $ \delta(0, Q_{\mu, \nu}) = \begin{cases}
0, \quad Q_{\mu, \nu} = 0 \\
1, \quad Q_{\mu, \nu} \neq 0 \\
\end{cases}$. Combining the objective in Equation~\ref{eq:vehicle_LP} and constraints in Equation~\ref{eq:vehicle_LP_cons}, we can formulate the LP hybrid dataset that maximizes the number of data and balances the contributions of data from different datasets as well as traffic scenarios.

The vehicle detection model is built on top of You Only Look Once (YOLO)-v5, a widely used object detection model \citep{yolov5}. YOLO-v5 is initially pre-trained, and we adopt the transfer learning scheme to inherit the pre-trained weights and tune the weight parameters on the LP hybrid dataset. The YOLO-v5 network is a general framework for detecting and classifying objects simultaneously. In the vehicle detection context, we only need to box out the vehicles from the background images regardless of vehicle types. Hence we reshape the output dimension into one with random initialized parameters. As the LP hybrid dataset contains camera images in various traffic scenarios, we can build a generalized detection model suitable for various traffic surveillance systems in different countries.

\section{Numerical Experiments}
\label{section:evaluation}
In this section, we conduct numerical experiments on the proposed camera calibration and vehicle detection methods to evaluate the performance in two traffic monitoring camera systems.

\subsection{Experimental Settings}
To demonstrate that the proposed framework can be applied to traffic density estimation in countries with different traffic surveillance systems, two case studies of traffic density estimation are conducted, one in Hong Kong (\texttt{HK}) and Sacramento, California (\texttt{Sac}) where the ground true data can be obtained at both sites. A comparison for these two cameras is shown in Table~\ref{tab:camera_comp}.
\begin{itemize}
    \item \texttt{HK}: camera images in Hong Kong are obtained from HKeMobility\footnote{\url{https://www.hkemobility.gov.hk/tc/traffic-information/live/cctv}} at the Chatham Road South, footbridge of The Hong Kong Polytechnic University, Kowloon, Hong Kong SAR, with the camera code K109F. Images containing seven vehicles are selected from June 22nd to June 25th, 2020. The resolution of images is $320 \times 240$ pixels.
    \item \texttt{Sac}: the camera images in California are obtained from Caltrans system\footnote{\url{https://cwwp2.dot.ca.gov/vm/iframemap.htm}} at I-50 Highway at 39 Street, Sacramento, CA, the US. Image containing six vehicles are selected from December 6th to December 7th, 2020. The resolution of images is $720 \times 480$ pixels.
\end{itemize} 

\begin{table}[h]
    \centering
    \begin{tabular}{p{0.20\columnwidth}p{0.35\columnwidth}p{0.35\columnwidth}}
    \hline
    \textbf{Attributes} & \texttt{HK} & \texttt{Sac} \\
    \hline
    Resolution & $320 \times 240$ pixels &  $720 \times 480$ pixels \\
    Update rate & 2 minutes &  1/30 seconds \\
    Orientation & Vehicle head & Vehicle tail \\
    Road type & Urban road & Highway \\
    Speed limit & 50 km/h & 105.3 km/h\\
    \hline
    \end{tabular}
    \caption{The comparison for the selected traffic cameras used for case studies in Hong Kong and Sacramento.}
    \label{tab:camera_comp}
\end{table}

For camera calibration, all vehicles are annotated with eight key points: left headlight, right headlight, front license plate center, front wiper center, left wing mirror, right wing mirror, back left corner and back right corner. Any key points not visible in an image are excluded. Besides, five popular vehicle models are involved with three-dimensional information: Toyota Corolla, Toyota Prius, Honda Civic, BMW Series 4 and Tesla Model S. The three-dimensional key points for those models are measured from the \textit{Dimensions}\footnote{\url{https://www.dimensions.com}}.
$\alpha$ in Equation~\ref{eq:obj_func} is set to $6$.

For vehicle detection, all of the datasets summarized in Table~\ref{tab:datasets_all} are incorporated. The ratio factors $\gamma$ and $\beta$ in Equation~\ref{eq:vehicle_LP_cons} are set to $0.25$. The LP hybrid  dataset is divided into a training set (80\%) and validation set (20\%). A total of 3,812 camera images are annotated to test the performance of the model trained on the LP hybrid dataset. 

All experiments are conducted on a desktop with Intel Core i9-10900K CPU @3.7GHz $\times$ 10, 2666MHz $\times$ 2 $\times$ 16GB RAM, GeForce RTX 2080 Ti $\times$ 2, 500GB SSD. 
The camera calibration and vehicle detection models are both  implemented with Python. For the camera calibration model, OpenCV \citep{itseez2015opencv} is used for computing Equation~\ref{eq:full_PnP_eq}  and running the EPnP algorithm \citep{Lepetit2009epnp}. In the candidate generation stage, the focal length is fixed at $350$ millimeters. The CMA-ES algorithm \citep{Hansen1996CMAES} is executed with the Nevergrad package \citep{nevergrad}. The numbers of iterations of CMA-ES in the vehicle model matching and parameter fine-tuning stage are set to 4,000 and 20,000, respectively. When tuning the vehicle detection model, we set the number of training epochs to $300$, and other hyperparameters take the default settings\footnote{\url{https://github.com/ultralytics/yolov5}}. The Adam optimizer \citep{kingmaB2014adam} is adopted with a learning rate of $0.001$.

\subsection{Experimental Results}

In this section, we compare the proposed camera calibration and vehicle detection models with existing baselines, respectively.

\subsubsection{Camera Calibration}
To evaluate the performance of the camera calibration method, we first compare the fine-tuning loss defined in Equation~\ref{eq:obj_func} among baseline models for the two cameras in \texttt{HK} and \texttt{Sac}. Based on the calibration results, we estimate the road length from the camera images, and the length estimated by each model is compared with the actual length.

To demonstrate the necessity of the three steps in \texttt{MVCalib}, Figure~\ref{fig:camera_calibration_loss} plots the fine-tuning loss defined in Equation~\ref{eq:obj_func} for the three stages: candidate generation, vehicle model matching and parameter fine-tuning. In particular, Figure~\ref{fig:3d_loss} includes the losses of all the vehicle index and vehicle model pairs for the first two stages, and Figure~\ref{fig:3d_loss_projection} plots the loss based on the matched vehicle model with the minimal fine-tuning loss. One can see that the fine-tuning loss defined in Equation~\ref{eq:3d_loss} decreases after each stage, which indicates that the CMA-ES can successfully reduce the loss in each stage. 

\begin{figure}[h]
    \centering
    \begin{subfigure}[b]{0.50\textwidth}
        \includegraphics[width=\textwidth]{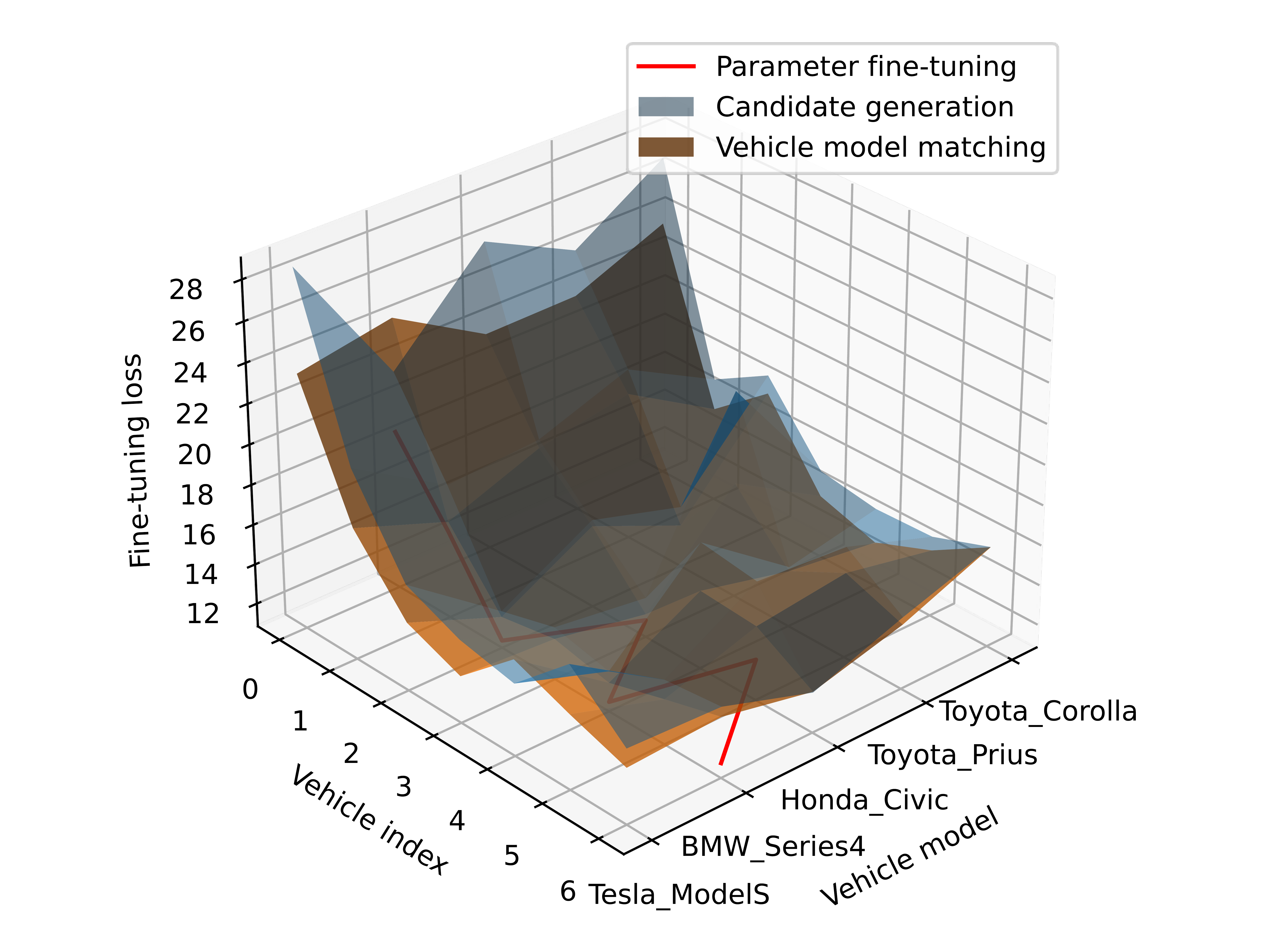}
        \caption{Fine-tuning losses with all indices and models.}
        \label{fig:3d_loss}
    \end{subfigure}
    \begin{subfigure}[b]{0.46\textwidth}
        \includegraphics[width=\textwidth]{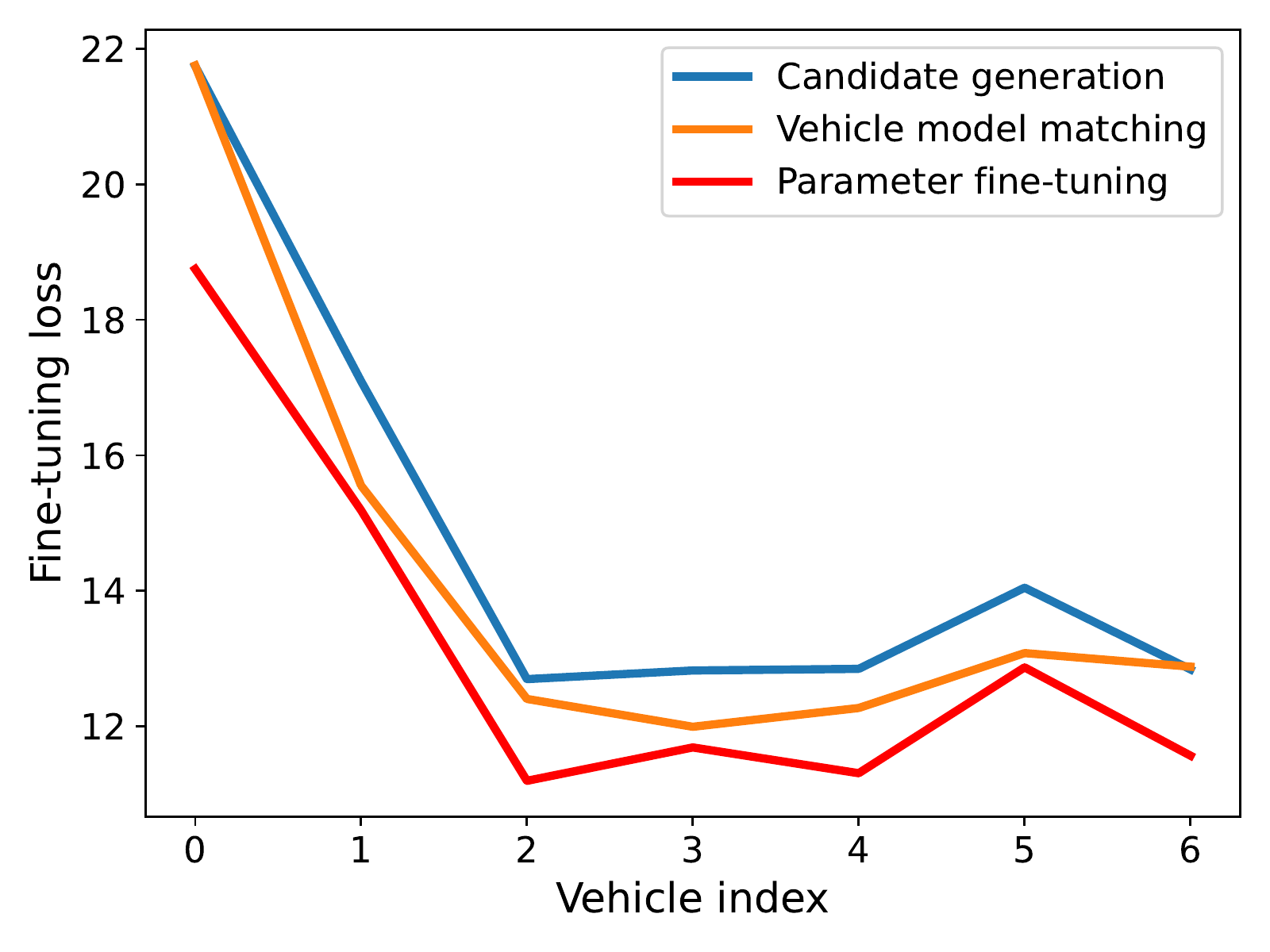}
        \caption{Minimal fine-tuning losses among all vehicle models for each vehicle.}
        \label{fig:3d_loss_projection}
    \end{subfigure}
    \caption{Fine-tuning losses with parameters in the three stages for camera calibration in \texttt{HK} (vehicle index is defined in Equation~\ref{eq:vehicle_index_def}).}
    \label{fig:camera_calibration_loss}
\end{figure}

We then measure the lengths of road markings on the camera images, as the road markings are invariant features on the road, and their lengths can be determined from measurements or official guide books. Detailed road marking information for the \texttt{HK} and \texttt{Sac} studies is shown in Figure~\ref{fig:roadmarks}. 
\begin{figure}[h]
    \centering
    \begin{subfigure}[b]{0.48\textwidth}
        \includegraphics[width=\textwidth]{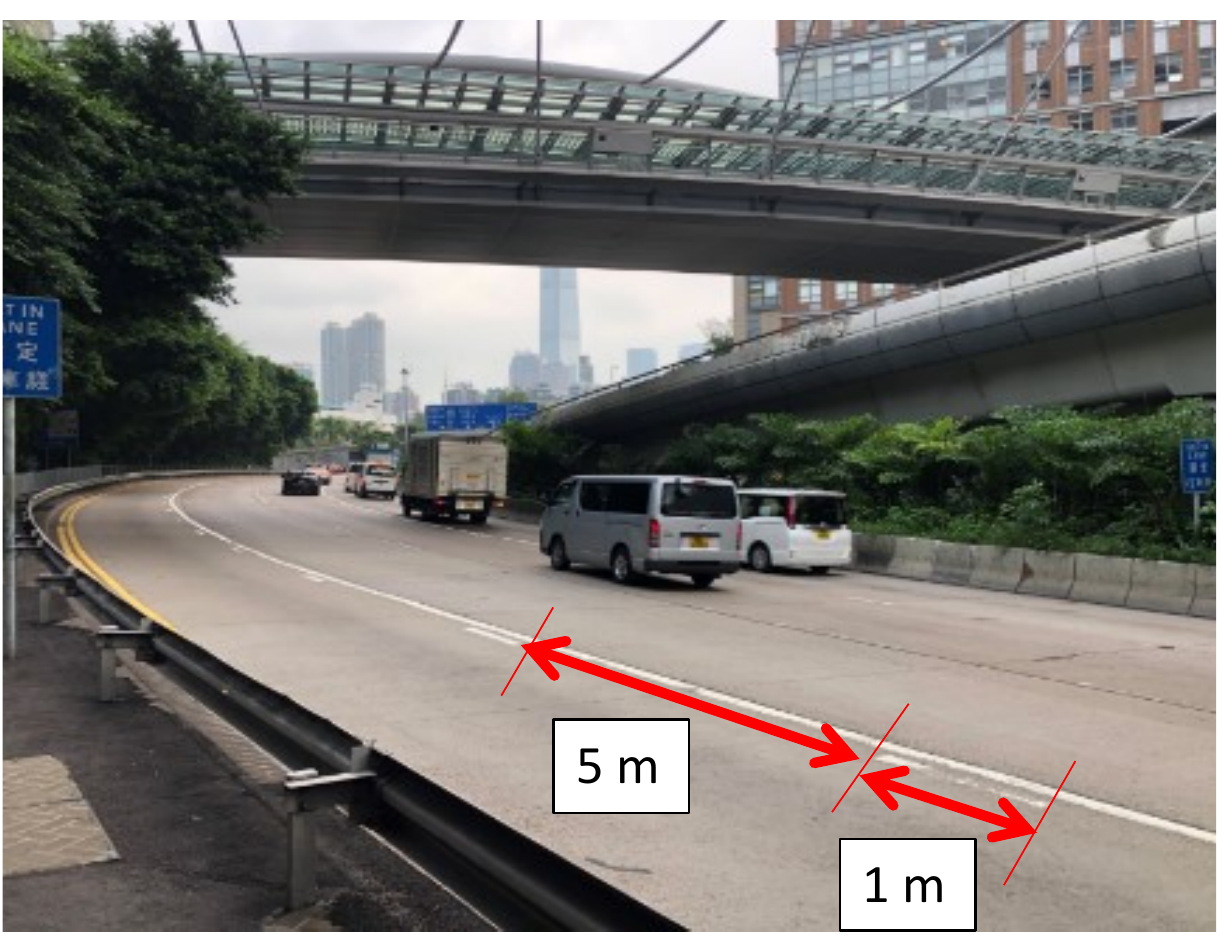}
        \caption{Real size of road markings in \texttt{HK}.}
        \label{fig:roadmarks_a}
    \end{subfigure}
    \begin{subfigure}[b]{0.48\textwidth}
        \includegraphics[width=\textwidth]{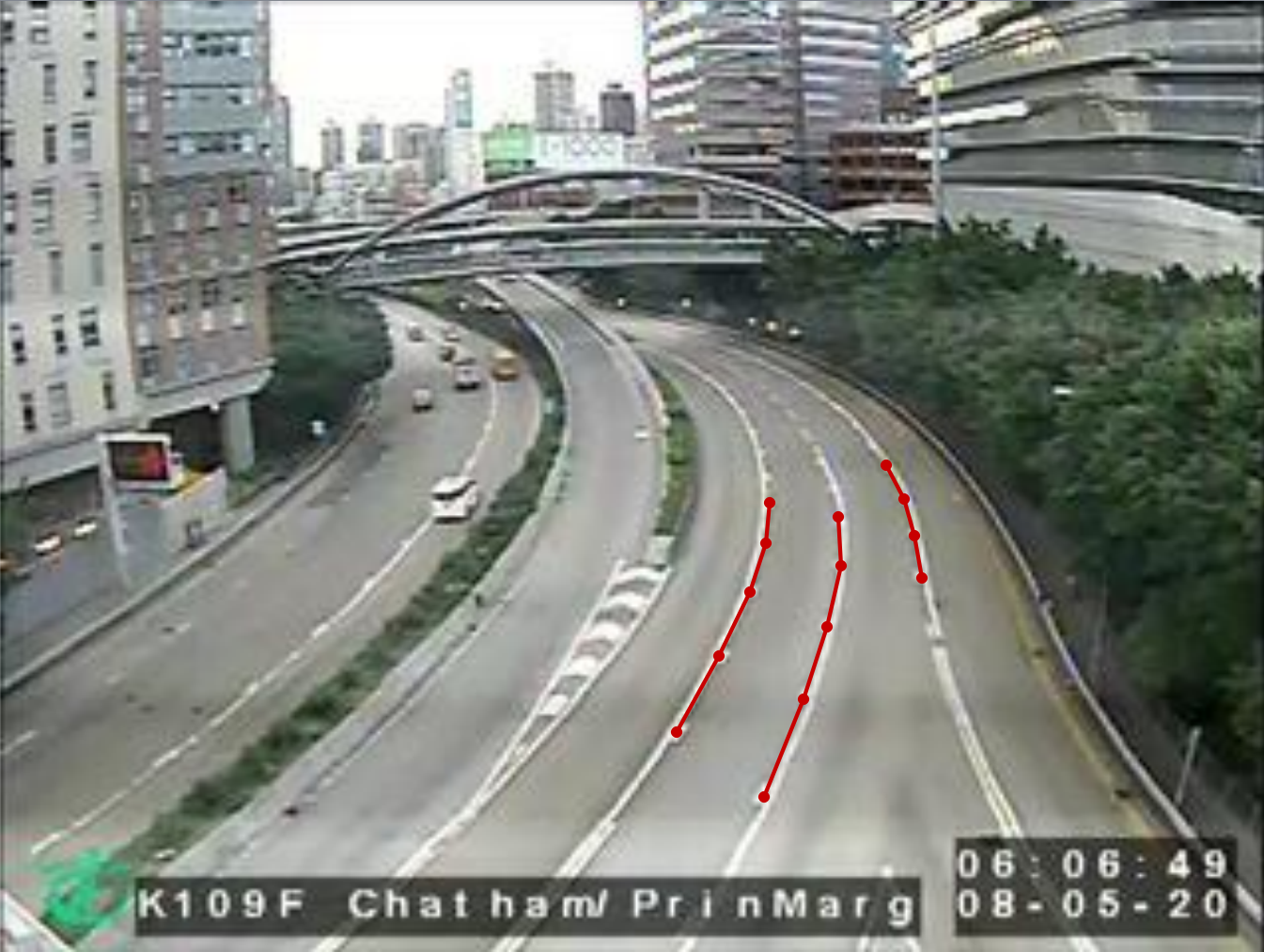}
        \caption{Testing points from camera images in \texttt{HK}.}
        \label{fig:roadmarks_b}
    \end{subfigure}
    \begin{subfigure}[b]{0.48\textwidth}
        \includegraphics[width=\textwidth]{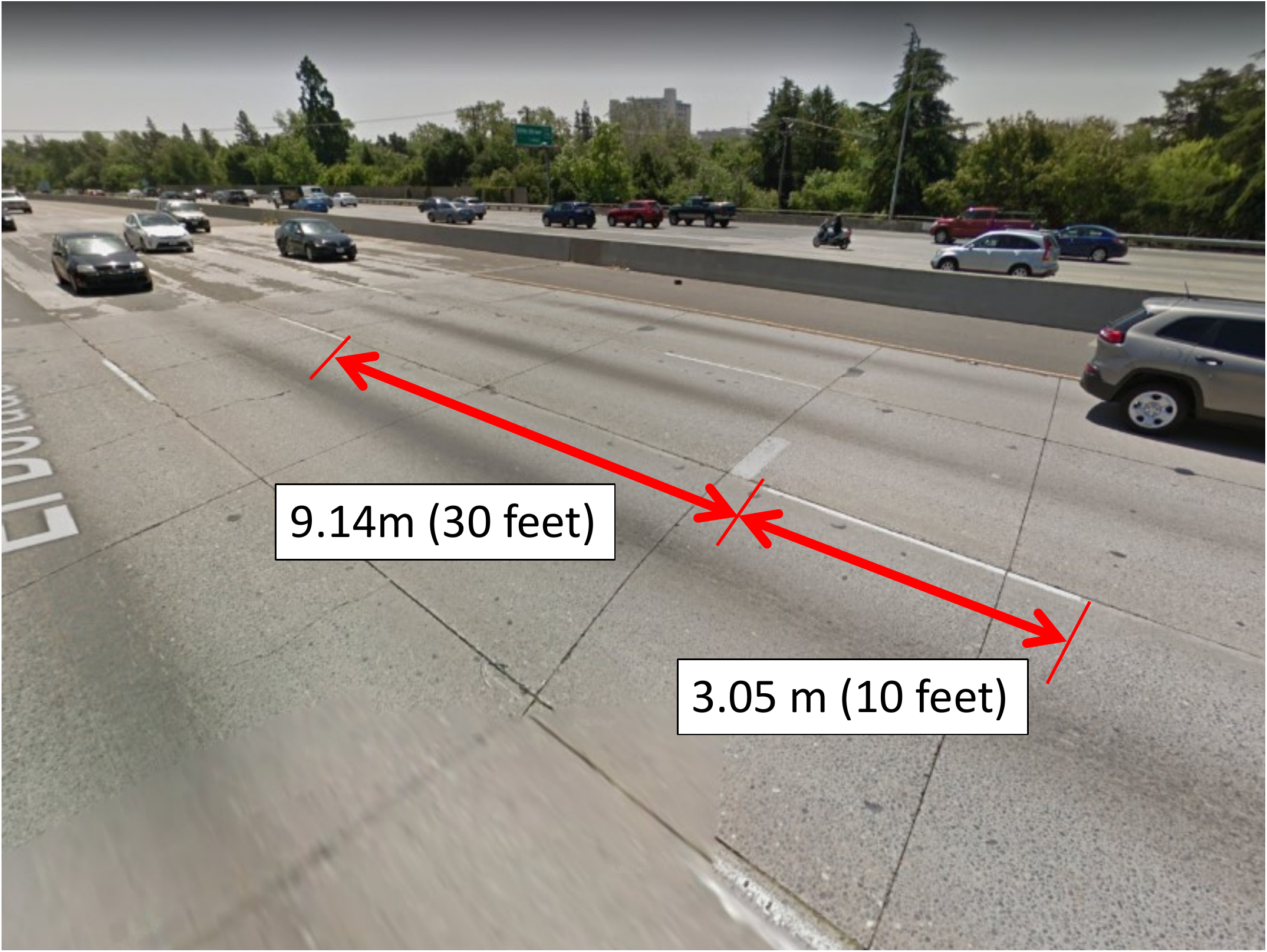}
        \caption{Real size of road markings in \texttt{Sac}.}
        \label{fig:roadmarks_c}
    \end{subfigure}
    \begin{subfigure}[b]{0.48\textwidth}
        \includegraphics[width=\textwidth]{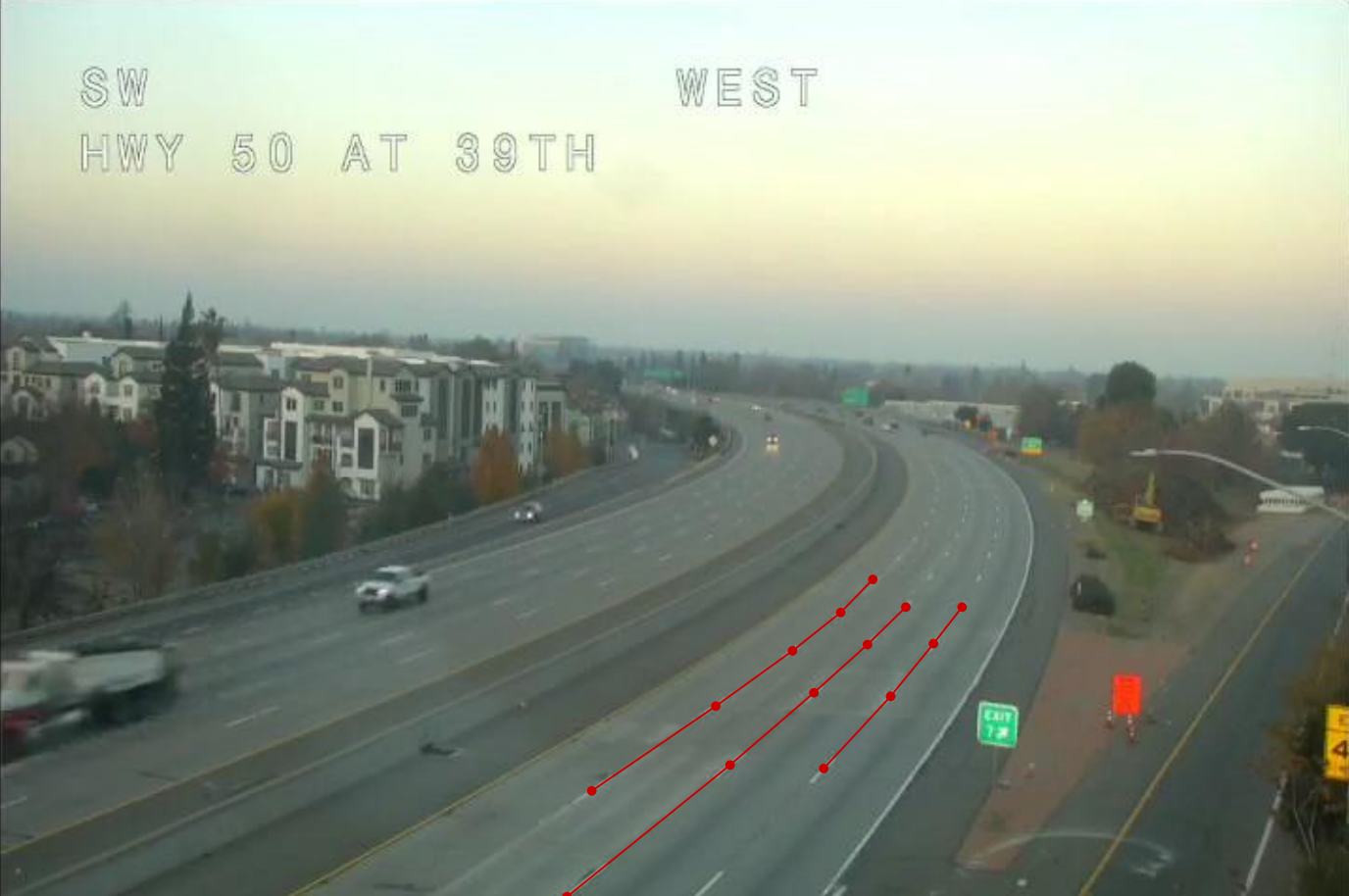}
        \caption{Testing points from camera images in \texttt{Sac}.}
        \label{fig:roadmarks_d}
    \end{subfigure}
    \caption{Driving lanes from real world and camera images in \texttt{HK} and \texttt{Sac}.}
    \label{fig:roadmarks}

\end{figure}
In Figure~\ref{fig:roadmarks_a}, the length of the white line is 1 meter and the interval between the white lines is 5 meters, which are obtained from field measurements. On the camera images, a total of 14 points are annotated at the midpoints of white lines, resulting in 12 line segments of the same length (shown in Figure~\ref{fig:roadmarks_b}). Hence each line segment corresponds to 6 meters in the real world. For the camera images in \texttt{Sac}, we likewise use the actual lengths of the lane markings on the I-50 Highway as the ground truth. According to the Manual on Uniform Traffic Control Devices (MUTCD) \citep{MUTCD2009}, the length of a white line is 10 feet (approximately 3.05 meters) and the interval is 30 feet (approximately 9.14 meters) (shown in Figure~\ref{fig:roadmarks_c}). On the camera images, we annotate 14 points resulting in 12 line segments (shown in Figure~\ref{fig:roadmarks_d}), elongated in 40 feet (approximately 12.19 meters) for each segment.

We compare our method with existing baseline models including \texttt{EPnP} \citep{Lepetit2009epnp}, \texttt{UPnP},  \texttt{UPnP+GN} (\texttt{UPnP} fine-tuned with the Gauss-Newton method) \citep{Penate2013upnp}, \texttt{GPnP} and \texttt{GPnP+GN} (\texttt{GPnP} fine-tuned with the Gauss-Newton method) \citep{Zheng2014gpnp}. 
The calibration results are shown in Table~\ref{tab:calibration_results_Sacramento}. The estimated lengths of the road markings on camera images with the actual lengths are used and three metrics are employed to compare different models: Rooted Mean Square Error (RMSE), Mean Absolute Error (MAE) and Mean Absolute Percentage Error (MAPE). At each stage of \texttt{MVCalib}, we compare its result with baseline methods in terms of their ability to solve the PnPf problem. To conduct an ablation study gauging the contribution of each stage, we run \texttt{MVCalib} with only the first stage (candidate generation), with the first two stages (up to vehicle model matching), and with all three stages. The three models are referred to as \texttt{MVCalib CG}, \texttt{MVCalib VM}, and \texttt{MVCalib}, respectively. In fact, the \texttt{MVCalib CG} is equivalent to the \texttt{EPnP} method.
\begin{table}[h]
    \centering
    \begin{tabular}{p{0.16\columnwidth}|p{0.08\columnwidth}p{0.08\columnwidth}p{0.09\columnwidth}|p{0.08\columnwidth}p{0.08\columnwidth}p{0.09\columnwidth}}
    \hline
    \multirow{2}*{\textbf{Method}} & \multicolumn{3}{c|}{\texttt{HK}} & \multicolumn{3}{c}{\texttt{Sac}} \\
      & \textbf{RMSE} & \textbf{MAE} & \textbf{MAPE}  & \textbf{RMSE} & \textbf{MAE} & \textbf{MAPE} \\
    \hline
    \texttt{UPNP}    & 25.80 & 22.21 & 370.03\% &  9.39 & 6.01 &  49.35\% \\
    \texttt{UPNP+GN}  & 2.02 & 0.62 & 10.36\% &  6.57 & 5.99 & 49.18\% \\
    \texttt{GPNP}     & 3.14  & 2.76 & 46.15\% &  6.59 & 4.87 & 39.96\% \\
    \texttt{GPNP+GN}  & 2.24 & 1.98 & 33.15\% & $>100$ & $>100$ & $>100$\% \\
    \texttt{MVCalib CG}  & 1.68 & 1.49 & 24.91\% &  2.30 & 1.94 & 15.97\% \\
    \texttt{MVCalib VM}  & 0.98 & 0.77 & 12.83\% &  \textbf{0.58} & 0.11 & 0.95\% \\
    \texttt{MVCalib}  & \textbf{0.55} & \textbf{0.20} & \textbf{3.36\%} &  \textbf{0.58} & \textbf{0.10} & \textbf{0.86\%} \\
    \hline
    \end{tabular}
    \caption{Comparison of results of monitoring camera calibrated by different methods in \texttt{HK} and \texttt{Sac} (unit for RMSE and MAE: meter).}
    \label{tab:calibration_results_Sacramento}
\end{table}

One can see from Table~\ref{tab:calibration_results_Sacramento} that \texttt{UPnP (GN)} and \texttt{GPnP (GN)} yield unsatisfactory solutions owing to the low image quality. As they take the focal length into account, the complexity of the problem is significantly increased, and hence they require high-resolution images, and more numerous and accurate annotation points. 

As for the ablation study, we compare \texttt{MVCalib CG}, \texttt{MVCalib VM}, and \texttt{MVCalib} to evaluate the contribution of each stage. In the vehicle model matching stage, if we optimize the focal length with other parameters simultaneously, the estimation results are greatly improved relative to \texttt{MVCalib CG}, demonstrating that the estimation of focal length is necessary and important for the calibration of traffic monitoring camera. In the full \texttt{MVCalib}, we also incorporate the joint information of multi-vehicle under the same camera. \texttt{MVCalib} achieves the best result among all models. For the monitoring camera in \texttt{HK}, the average error is only approximately 20 centimeters for estimating the six-meter road markings, less than 5\% in MAPE. while in \texttt{Sac}, the average error is only 10 centimeters for the forty-foot road markings, less than 1\% in MAPE. 

Besides, \texttt{MVCalib} outperforms the other models in terms of all three metrics, which means that the calibration results are close to the ground truth. Snapshots of calibration results of monitoring cameras in \texttt{HK} and \texttt{Sac} are shown in Figure~\ref{fig:snapshot_calibration}, where the distance between any two red dots is one meter. 
\begin{figure}[h]
    \centering
    \begin{subfigure}[b]{0.48\textwidth}
        \includegraphics[width=\textwidth]{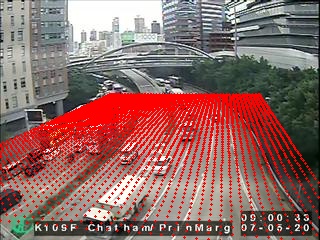}
        \caption{The snapshots of calibration result in \texttt{HK}.}
        \label{fig:snapshot_hk}
    \end{subfigure}
    \begin{subfigure}[b]{0.48\textwidth}
        \includegraphics[width=\textwidth]{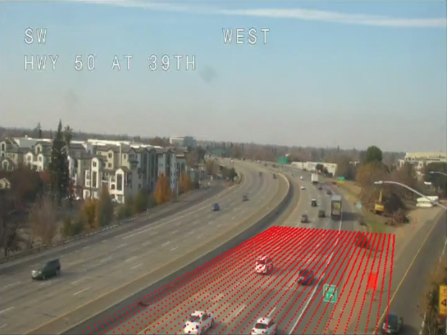}
        \caption{The snapshot of calibration result in \texttt{Sac}.}
        \label{fig:snapshot_sacramento}
    \end{subfigure}
    \caption{The snapshot of calibration results in \texttt{HK} and \texttt{Sac}.}
    \label{fig:snapshot_calibration}
\end{figure}
Owing to the perspective effect, the distance between red dots on images appear closer when they are more distant from the camera. Through visual inspection, we note that the estimation of focal length is reasonable and the skew of perspective error is small.

\subsection{Vehicle Detection}
In the detection model, two traffic scenarios are considered: daytime and nighttime. A total of 76,898 images are hybridized in the LP-hybrid detection dataset after solving for the LP in Equation~\ref{eq:vehicle_LP} and \ref{eq:vehicle_LP_cons}. The detailed allocation of the 76,898 images is presented in  Table~\ref{tab:dataset_detection}. 
\begin{table}[h]
    \centering
    \begin{tabular}{p{0.17\columnwidth}|p{0.26\columnwidth}p{0.26\columnwidth}p{0.2\columnwidth}}
        \hline 
        \textbf{Dataset} & \textbf{\# images in daytime} & \textbf{\# images at nighttime} & \textbf{Total \# images} \\ 
        \hline
        BDD100K & 8,319 & 8,398 & 16,717 \\
        BITVehicle & 7,325 & 0 & 7,325 \\
        CityCam & 8,459 & 0 & 8,459 \\
        COCO & 7,111 & 7,619 & 14,730 \\
        MIO-TCD-L & 8,892 & 7,413 & 16,305 \\
        UA-DETRAC & 7,955 & 5,407 & 13,362 \\
        \textbf{Total} & \textbf{48,061} & \textbf{28,837} & \textbf{76,898} \\
        \hline
    \end{tabular}
    \caption {Allocation of the LP hybrid dataset.} 
    \label{tab:dataset_detection}
\end{table}

To evaluate the generalizability of the vehicle detection model trained on the LP hybrid dataset, we also train the YOLO-v5 individually with the BDD100K, BITVehicle, CityCam, COCO, MIO-TCD-L, and UA-DETRAC datasets for benchmark comparison. Additionally, an integrated dataset incorporating all of the aforementioned datasets without balancing the numbers of images in the daytime and nighttime is also considered, called the Spaghetti dataset, is also compared. For the model trained on each dataset, we report the vehicle detection accuracy on the testing data. Several metrics are used in evaluating the performance of the vehicle detection models, including precision, recall, AP@0.5, and AP@0.5:0.95. Interpretation about these metrics is shown in \ref{apx:explain_metrics}.

\begin{table}[h]
    \begin{tabular}{p{0.15\columnwidth}|p{0.13\columnwidth}p{0.13\columnwidth}p{0.14\columnwidth}p{0.14\columnwidth}p{0.14\columnwidth}}
    \hline
    \textbf{Name} & \textbf{Precision} & \textbf{Recall} & \textbf{mAP@0.5} & \textbf{mAP@0.5:0.95} & \textbf{Dataset size} \\
    \hline
    BDD-100K & 0.361 & 0.364 & 0.326 & 0.144 & 100,000 \\
    BITVehicle & 0.255 & 0.009 & 0.062 & 0.035 & 9,850 \\
    CityCam & 0.412 & 0.938 & 0.881 & 0.538 & 60,000 \\
    COCO & \textbf{0.978} & 0.017 & 0.556 & 0.340 & 17,684 \\
    MIO-TCD-L & 0.737 & 0.885 & 0.899 & 0.578 & 137,743 \\
    Pretrained & 0.455 & 0.899 & 0.838 & 0.552 & 0 \\
    UA-DETRAC & 0.775 & 0.693 & 0.758 & 0.488 & 138,252 \\
    Spaghetti & 0.605 & 0.948 & \textbf{0.927} & \textbf{0.608} & \textbf{434,993} \\
    \textbf{LP hybrid} & 0.583 & \textbf{0.949} & 0.921 & 0.594 & 76,898 \\
    \hline
    \end{tabular}
    \caption{Evaluation results for different detection models on images in daytime.}
    \label{tab:detection_results_day}
\end{table}
\begin{table}[h]
    \begin{tabular}{p{0.15\columnwidth}|p{0.13\columnwidth}p{0.13\columnwidth}p{0.14\columnwidth}p{0.14\columnwidth}p{0.14\columnwidth}}
    \hline
    \textbf{Name} & \textbf{Precision} & \textbf{Recall} & \textbf{mAP@0.5} & \textbf{mAP@0.5:0.95} & \textbf{Dataset size} \\
    \hline
    BDD-100K & 0.443 & 0.316 & 0.302 & 0.124 & 100,000 \\
    BITVehicle & 0.058 & 0.001 & 0.018 & 0.010 & 9,850 \\
    CityCam & 0.402 & 0.793 & 0.713 & 0.412 & 60,000 \\
    COCO & \textbf{0.949} & 0.003 & 0.397 & 0.223 & 17,684 \\
    MIO-TCD-L & 0.805 & 0.746 & 0.817 & 0.511 & 137,743 \\
    Pretrained & 0.387 & 0.862 & 0.781 & 0.471 & 0 \\
    UA-DETRAC & 0.708 & 0.573 & 0.629 & 0.365 & 138,252 \\
    Spaghetti & 0.689 & 0.872 & 0.882 & \textbf{0.546} & \textbf{434,993} \\
    \textbf{LP hybrid} & 0.653 & \textbf{0.89} & \textbf{0.886} & 0.545 & 76,898 \\
    \hline  
    \end{tabular}
    \caption{Evaluation results for different detection models on images at nighttime.}
    \label{tab:detection_results_night}
\end{table}
Tables~\ref{tab:detection_results_day} and \ref{tab:detection_results_night} present the evaluation results for the models trained with the LP hybrid and other datasets for daytime and nighttime, respectively.
The model trained on the COCO dataset reaches the highest precision, while its recall is less than 2 percent, meaning that the model is highly confident in detecting a small portion of vehicles from camera images, but also tends to miss many vehicles. The model trained on the LP hybrid dataset reaches the highest recall and also achieves an acceptable precision rate. For the metrics of mAP@0.5 and mAP@0.5:0.95, the model trained on the Spaghetti dataset achieves the best performance, but the gap between the models trained on the Spaghetti dataset and the LP hybrid dataset for mAP@0.5 is less than 1\% and the gap for mAP@0.95 is less than 2\%. For images at nighttime, the model on the LP hybrid dataset outperforms that trained on the Spaghetti dataset on mAP@0.5, indicating that the proposed LP hybrid dataset can improve the detection performance at night. Moreover, as there are fewer than 80,000 images in the LP hybrid dataset, but more than 400,000 images in the Spaghetti dataset, it takes only 6 days to train a model on the LP hybrid dataset, while the training time on the Spaghetti dataset is beyond 21 days.

\section{Case Study I: monitoring cameras in Hong Kong}
\label{section:case_study_1}

In this section, we conduct a case study of traffic density estimation using camera images on the Chatham Road South, underneath the footbridge of the PolyU, Hong Kong SAR. Given the study region, we divide the roads into four lanes (numbered along the x-axis), and define vehicle locations along the y-axis, as shown in Figure~\ref{fig:K109F_lane_split}. 
\begin{figure}[h] 
    \centering 
    \includegraphics[width=0.5\textwidth]{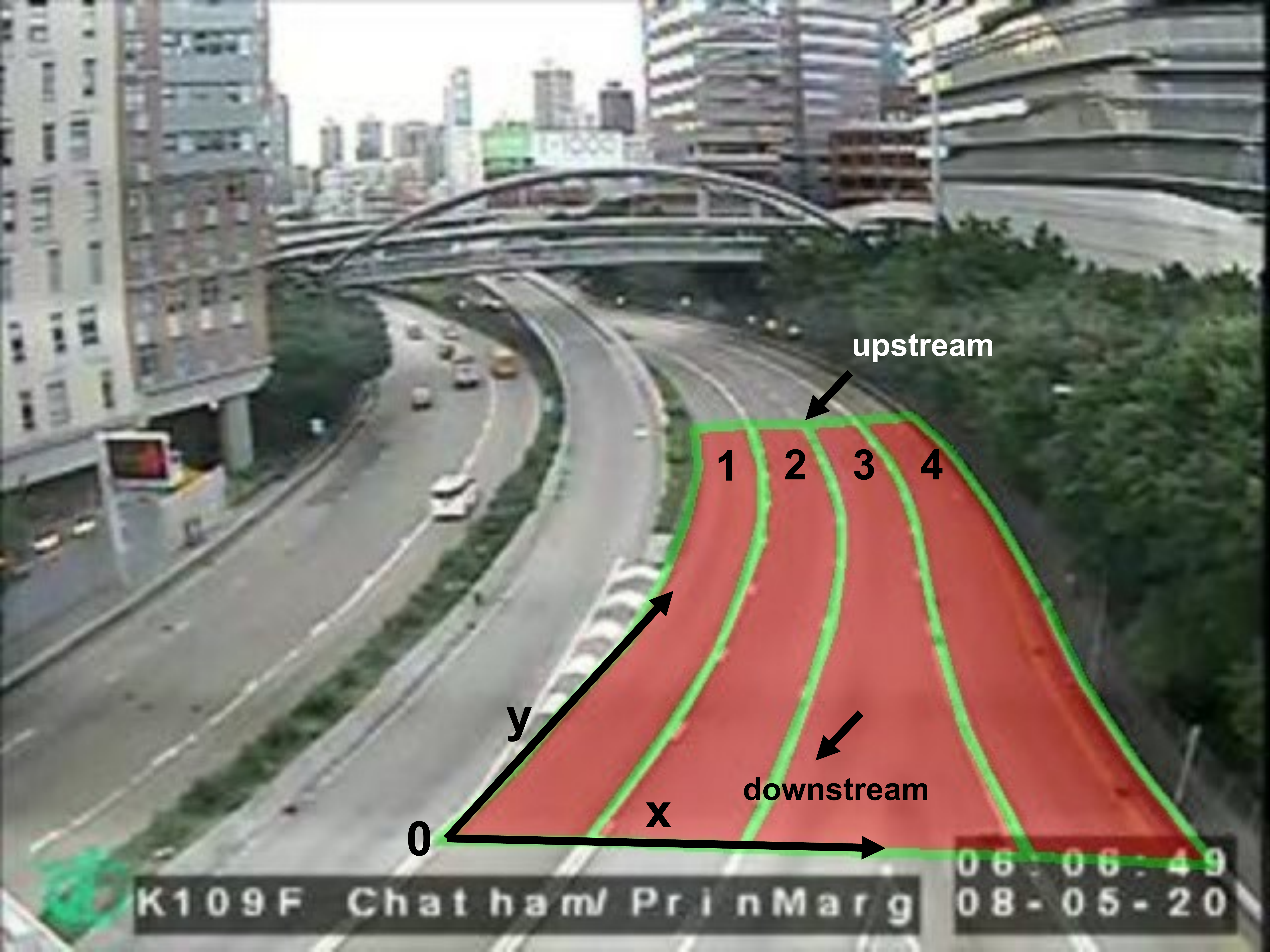}
    \caption{Driving lanes for traffic density estimation beneath the monitoring camera.} 
    \label{fig:K109F_lane_split} 
\end{figure}

The length of each lane can be estimated from the images using the calibration results, and the number of vehicles can be counted using the vehicle detection model. The traffic density in each lane can be estimated by dividing the number of vehicles by the length of each lane at each location and time point. To evaluate the estimated density, a high-resolution ($1920 \times 1080$ pixels per frame) camera is installed shooting the same region with different directions, and the camera video is acquired in this case study as a ground truth. The video recorded by this camera, showing the traffic conditions over 21 hours from 11:30 PM, September 23th to 8:30 PM, September 24, 2020.

\begin{figure}[h]
    \centering
    \includegraphics[width=0.95\textwidth]{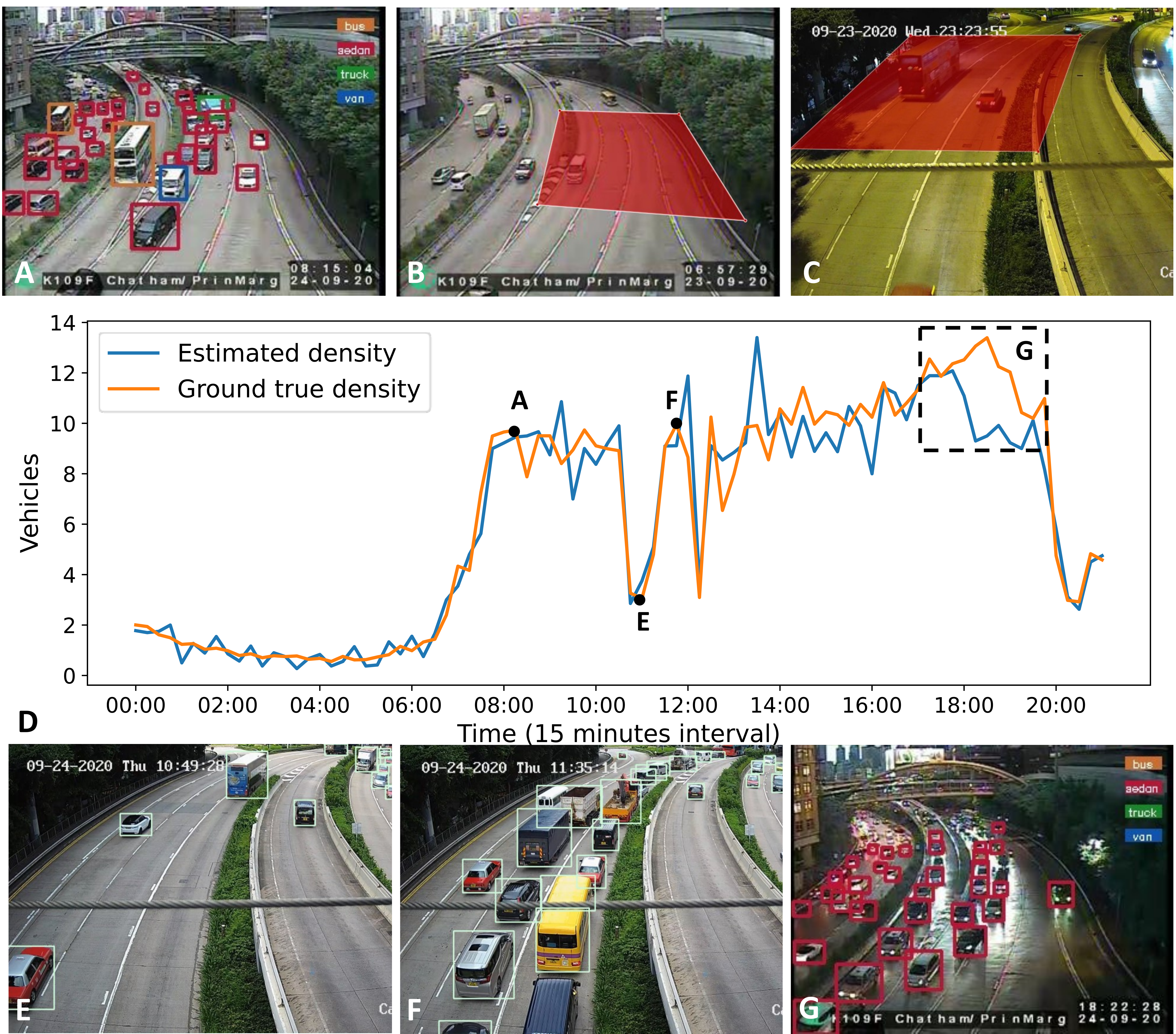}
    \caption{Overview of the vehicle detection results in \texttt{HK}.}
    \label{fig:case_study_HK_montage}
\end{figure}

\subsection{Estimation accuracy}

An overview of the vehicle detection results is presented in Figure~\ref{fig:case_study_HK_montage}. Figure~\ref{fig:case_study_HK_montage}A displays a snapshot of vehicle detection using the model trained on the LP hybrid dataset of images taken in daytime. By boxing out identical study regions in the traffic monitoring camera images and high-resolution videos (shown in Figure~\ref{fig:case_study_HK_montage}B and \ref{fig:case_study_HK_montage}C), the estimated number of vehicles can be compared with the ground truth in Figure~\ref{fig:case_study_HK_montage}D. We select four points or regions in Figure~\ref{fig:case_study_HK_montage}D, which are shown in Figure~\ref{fig:case_study_HK_montage}A, E, F and G. Figure~\ref{fig:case_study_HK_montage}A shows the beginning of the morning peak when the vehicle number significantly increases. Lanes \#1 and \#2 in the study region (numbered from the left) become visibly congested in the camera images. Points E and F are a pair of points that depicts contrasting traffic conditions when the traffic density fluctuates dramatically in a short time interval. If we inspect images taken around 11:00 AM and 11:30 AM, respectively on September 24, 2020, which are the corresponding points E and F. In Figure~\ref{fig:case_study_HK_montage}E, it can be seen that there are few vehicles on the road, and hence the traffic density is relatively low at point E. However, at point F, there is a sharp increase in the demand on the road. The traffic condition oscillates owing to the traffic signals downstream, which causes the pronounced changes between points E and F. Figure~\ref{fig:case_study_HK_montage}G depicts the traffic conditions at the evening peak when the vehicle number reaches the daily maximum. The evening peak fades away quickly and disappears at approximately 8:00 PM.

Compared to the estimated and ground true traffic density, the developed model succeeded in tracking the growth of the morning peak and detecting the fluctuation of traffic conditions. However, at point G, some of the vehicles are miss-detected in the evening peak. This may have been caused by dazzles from the headlights and the light reflected from the ground, which make it difficult for the vehicle detection model to identify the features of vehicles. This phenomenon is a common issue in Computer Vision  (CV), which will be left for future research. Overall, the estimated result is close to the ground truth most of the time, which demonstrates that the detection model can accomplish an accurate detection despite the low-resolution and low-frame-rate of the images.

The RMSE, MAE, and MAPE of the estimated traffic density for each lane and the entire road are presented in Table~\ref{tab:K109F_density}, and a comparison of traffic densities from estimation and the ground truth is shown in Figure~\ref{fig:case_study_hk_lane_density}. 
\begin{table}[h]
    \begin{tabular}{p{0.20\columnwidth}|p{0.23\columnwidth}p{0.23\columnwidth}p{0.23\columnwidth}}
    \hline
    \textbf{Lane ID} & \textbf{RMSE} & \textbf{MAE} & \textbf{MAPE} \\
    \hline
    Lane \#1 & 16.94 & 12.65 & 19.60\% \\
    Lane \#2 & 12.98 & 9.23 & 27.48\% \\
    Lane \#3 & 11.11 & 7.77 & 41.24\% \\
    Lane \#4 & 8.70 & 6.53 & 50.44\% \\
    Average & 12.43 & 9.04 & 34.69\% \\
    \hline
    \end{tabular}
    \caption{The RMSE, MAE and MAPE of the estimated density for different lanes from monitoring cameras in \texttt{HK} (unit: RMSE, MAE: veh/km/lane).}
    \label{tab:K109F_density}
\end{table}

\begin{figure}[h]
    \centering 
    \includegraphics[width=1\textwidth]{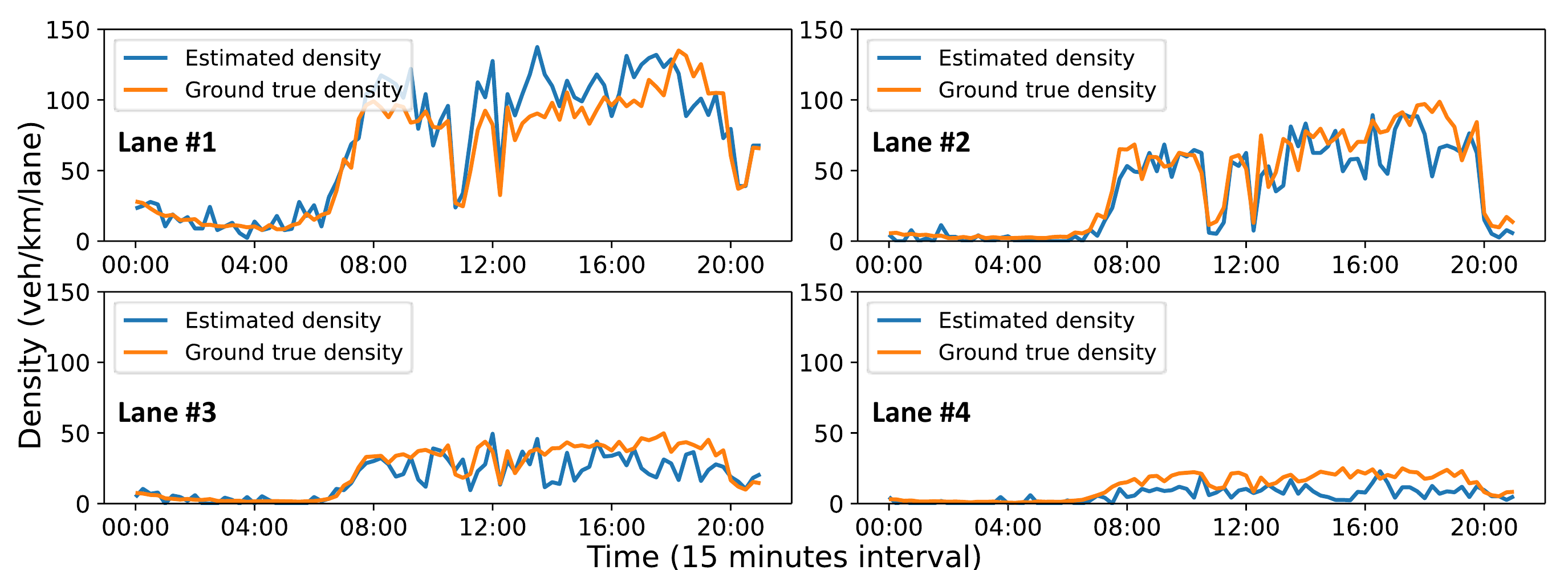}
    \caption{Traffic density from 0:00 AM to 9:00 PM in \texttt{HK}}
    \label{fig:case_study_hk_lane_density}
\end{figure}

One can see that the estimated density approximates the actual density, and the density fluctuation is accurately captured. The MAPE is relatively high because this metric is sensitive when the density is small. For example, if the true density is 2 veh/km/lane, while the estimated density is 1 veh/km/lane, then the MAPE is 50\%.
The traffic density of Lane \#1 is overestimated with an MAE of approximately 12 veh/km/lane, while the traffic density of Lanes \#2, \#3 and \#4 are underestimated with an MAE of approximately 9, 7, and 6 veh/km/lane, respectively. 
The possible causes of the under estimation and over estimation are two-fold: 1) The frame rate is not sufficient enough to support an individual estimation for each lane. Since the image will be updated once every two minutes, an average of 7.5 images will be accumulated in a time interval of 15 minutes. The estimation may result in a biased estimation since the small-size samples happen to capture the non-recurrent patterns of the traffic density. 2) The determination of the lane of each vehicle may be biased, as the lane occupied by each vehicle is determined by the center of the bounding box. When the road is curved in images and the vehicle is large, the center of the bounding box may shift to another lane, affecting the accuracy of the estimations in both lanes.

\subsection{Spatio-temporal patterns of the density}
Using the developed density estimation framework, we estimate the 24/7 density for the same study region from June 22 to September 22, 2020. The traffic density is estimated on each day and at each location, and we can also average the spatio-temporal density by time and by location.

 Figure~\ref{fig:his_veh_distribution} shows the spatio-temporal distribution of the traffic density in the study region on an averaged day. Note the data in averaged day is the average data from June 22 to September 22, 2020, where a lighter color represents a higher traffic density. 
\begin{figure}[h]
    \centering 
    \includegraphics[width=0.98\textwidth]{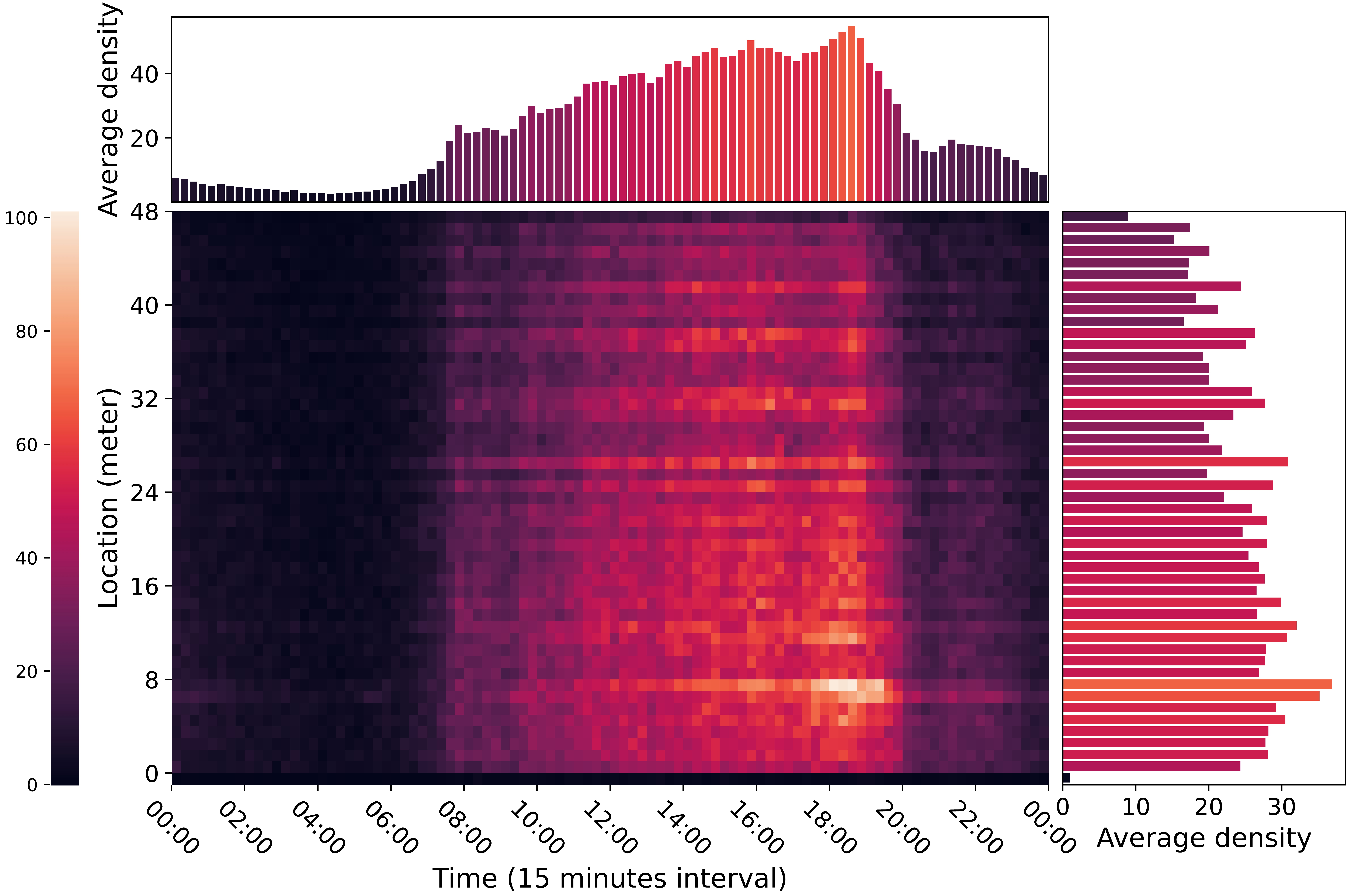}
    \caption{The spatio-temporal distribution of traffic density on an average day (unit: veh/km/lane) in \texttt{HK}.} 
    \label{fig:his_veh_distribution}
\end{figure}
One can see that the morning peak usually starts at approximately 7:00 AM when the number of vehicles begin to grow rapidly. After 8:00 AM, the number of vehicles continues to grow, but more slowly, until 1:00 PM, when the traffic density briefly decreases. The traffic density then grows again until 5:00 PM where another brief decrease of traffic density occurs. The density then reaches the highest level at the evening peak at approximately 6:30 PM. Finally, the traffic density decreases significantly and falls to a low level (approximately 10 veh/km/lane) after 9:00 PM.

We further visualize the density variations during one week. Figure~\ref{fig:his_density_week} presents the average traffic density among all lanes on each day of the week and at each time of day.
\begin{figure}[h] 
    \centering 
    \includegraphics[width=0.98\textwidth]{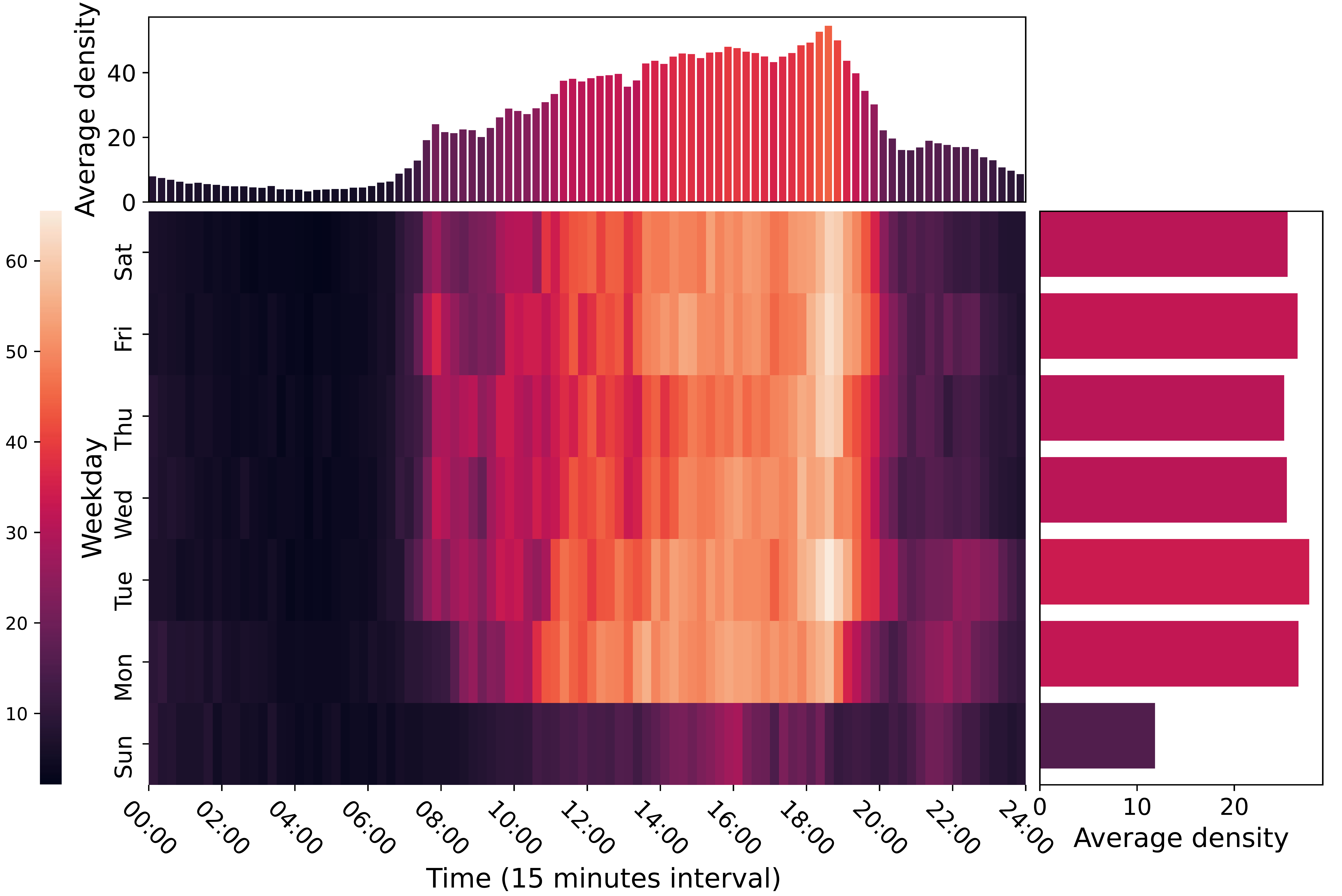}
    \caption{The distribution of traffic density in terms of the day of week and time of the day in \texttt{HK}.} 
    \label{fig:his_density_week}
\end{figure}
As shown in Figure~\ref{fig:his_density_week}, the traffic density shows the similar patterns from Monday through Saturday but it drops considerably on Sunday. This is consistent with the fact that the Hong Kong Government stipulates that only Sunday is a public holiday every week, and some companies require employees to work for five-and-a-half or six days a week. Hence, the traffic density on Saturday is similar to other workdays. On Sunday, fewer people choose to travel, resulting in decreased traffic density on the road. Moreover, since the travel demand decreases on Sunday, the morning and evening peaks also disappear while the average density is only a half compared to that on workdays. On Monday, the morning and evening peaks return, but the heatmap shows that the morning peak is postponed by an hour while the evening peak is advanced by an hour. As the average daily traffic density is similar to that on other weekdays, this may indicate that the traffic congestion will be more serious on Mondays as more vehicles will occupy the road during peak hours.

 
\begin{figure}[h]
    \centering
    \includegraphics[width=1\textwidth]{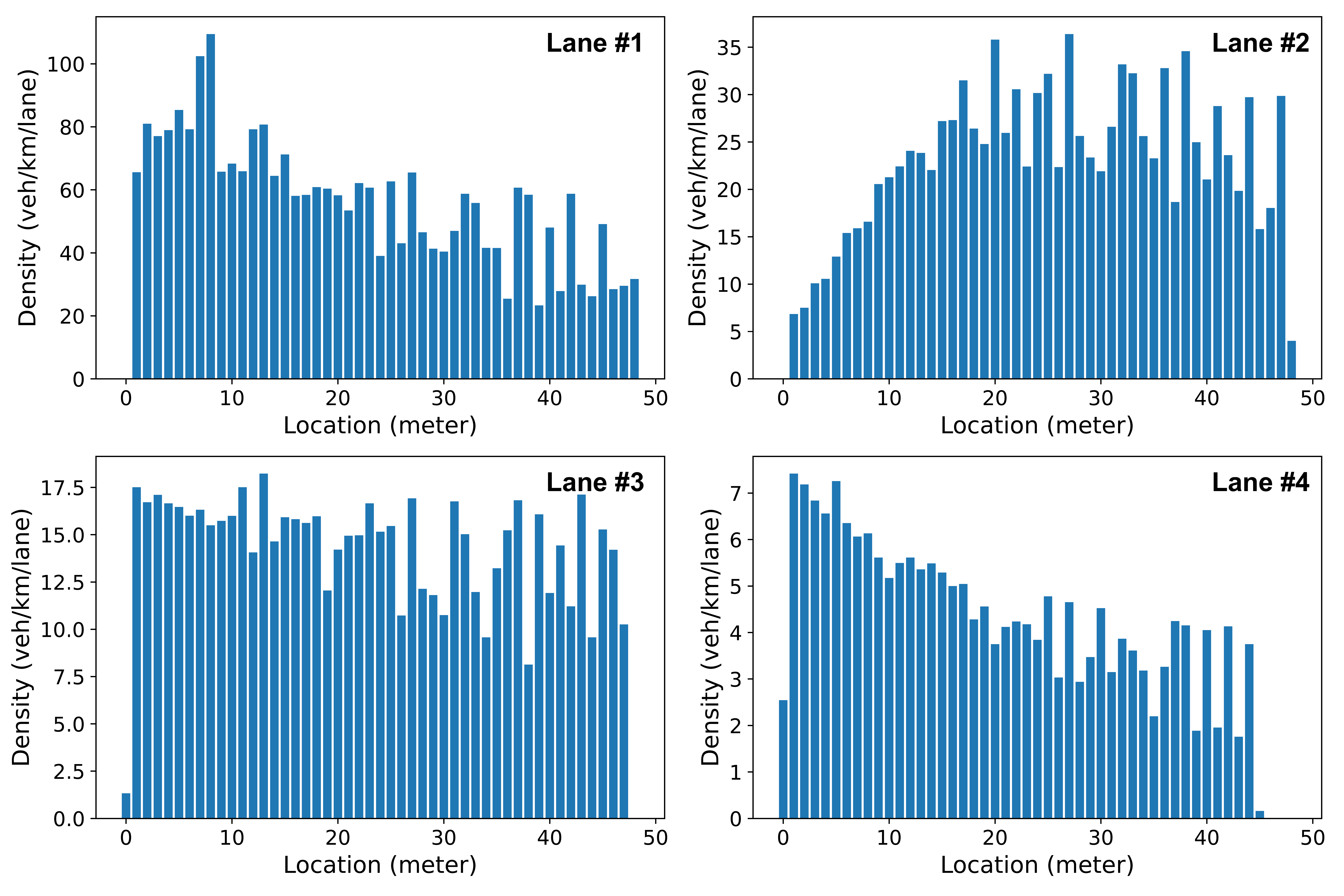}
    \caption{Average density along the road for each lane in \texttt{HK}.}
    \label{fig:location_hist_HK}
\end{figure}

We further plot the average density along the road for each lane separately, as shown in Figure~\ref{fig:location_hist_HK}.
Evidently, the spatial distribution of the density in \texttt{HK} is not uniform. For Lanes \#1, \#2, and \#4, vehicles are more prone to accumulate at the downstream sections of roads than upstream, possibly because of a bottleneck at the downstream. Additionally, the density of Lane \#1 is much higher than those of the other lanes,  because many vehicles use Lane \#1 to merge right at the downstream.

\subsection{Calibration of the fundamental diagrams}

Combining the traffic speed data for the same time interval, the fundamental diagrams can be calibrated using the estimated density data. While the traffic density data can be obtained from traffic monitoring cameras in Hong Kong, the traffic speed data for the same road can be acquired from the Traffic Speed Map published by the Hong Kong Transport Department \citep{traffic_speed_map}. Figure~\ref{fig:fd_hk} demonstrates the calibrated relationship between density-speed and density-flow. 
\begin{figure}[h] 
    \centering 
    \includegraphics[width=0.98\textwidth]{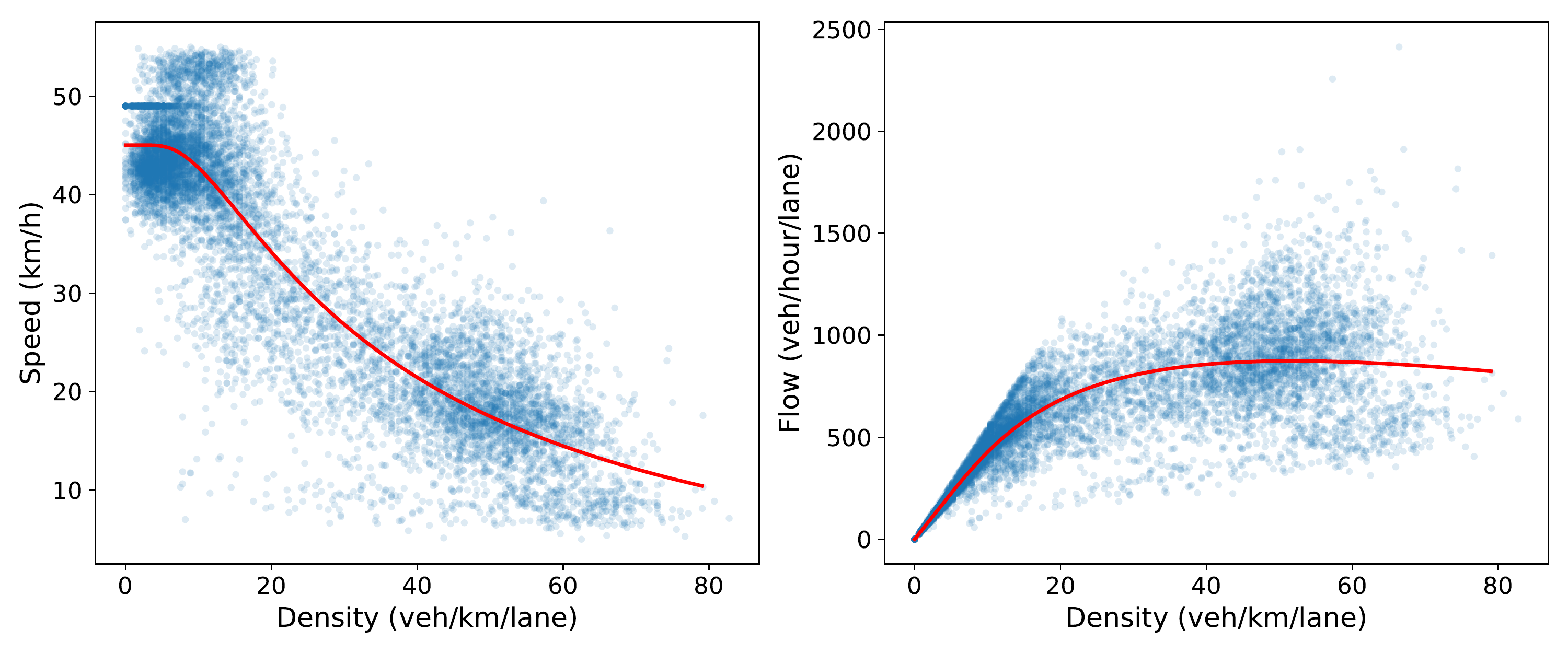}
    \caption{The calibrated fundamental diagrams in \texttt{HK}.} 
    \label{fig:fd_hk}
\end{figure}
It has been shown that the average speed decreases with increasing traffic density. If we use the Newell model \citep{newell} to fit the density-speed and density-flow maps (shown in Equation~\ref{eq:newell_model}), the result is shown in Figure~\ref{fig:fd_hk} using red line, and the parameter for free flow speed $\overline{v}_f = 45.03$km/h, jam density $k_j=239.86$ veh/km/lane and the slope of the speed-spacing curve $\lambda = 1396.43$, respectively. The maximum flow rate in Figure~\ref{fig:fd_hk} is $873.79$ veh/hour/lane.
\begin{equation}
    \overline{v} = \overline{v}_f \left( 1 - \exp \left( -\frac{\lambda}{\overline{v}_f} \left( \frac{1}{k} - \frac{1}{k_j} \right) \right) \right)
    \label{eq:newell_model}
\end{equation}

\clearpage
\section{Case Study II: monitoring cameras in Sacramento}
\label{section:case_study_2}

To demonstrate the generalizability of the proposed framework, another case study is conducted using a monitoring camera in the Caltrans system. 
The monitoring video data is collected from the camera on the I-50 Highway at 39th Street, Sacramento, CA (shown in Figure~\ref{fig:case_study_cal_location} left).
\begin{figure}[h] 
    \centering
    \includegraphics[width=1\textwidth]{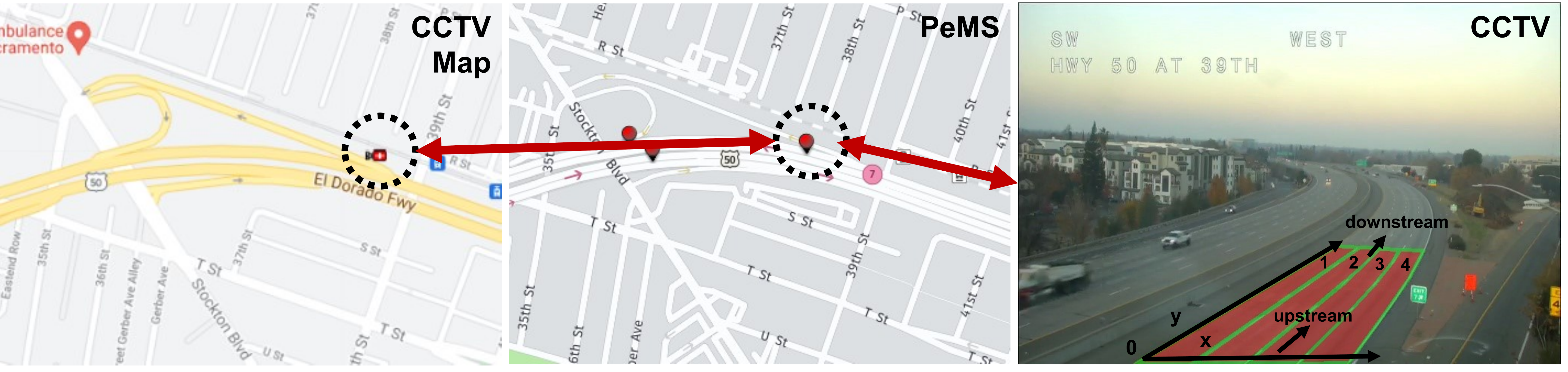}
    \caption{The location of the matched monitoring camera and double-loop detector in Caltrans.}
    \label{fig:case_study_cal_location}
\end{figure}

A 26-hour video is downloaded, covering the period from 2:30 AM on December 6 to 5:00 AM on December 7, 2020. Similar to the procedures for \texttt{HK}, key points on vehicles are annotated manually for camera calibration. The ground true density data are obtained from a double-loop detector at the same location (shown in Figure~\ref{fig:case_study_cal_location} center) within the same time period. The detector data are obtained from the PeMS system, which include the average traffic speed, density, and flow data. Given the study region,  we can also divide the roads into four lanes (numbered along the x-axis), and define vehicle locations along the the y-axis, as shown in Figure~\ref{fig:case_study_cal_location} right.

\subsection{Estimation accuracy}

The accuracy of the estimated traffic density is shown in Figure~\ref{fig:traffic_density_cal}. It can be seen that the traffic density  in \texttt{Sac} is much lower than that in \texttt{HK}, meaning that the congestion is less frequent in \texttt{Sac}. As the video is recorded on Sunday when most people do not go to work, there is only one peak of the traffic density during the 24-hour period. From Table~\ref{tab:Sacramento_density},
\begin{figure}[h]
    \centering
    \includegraphics[width=1\textwidth]{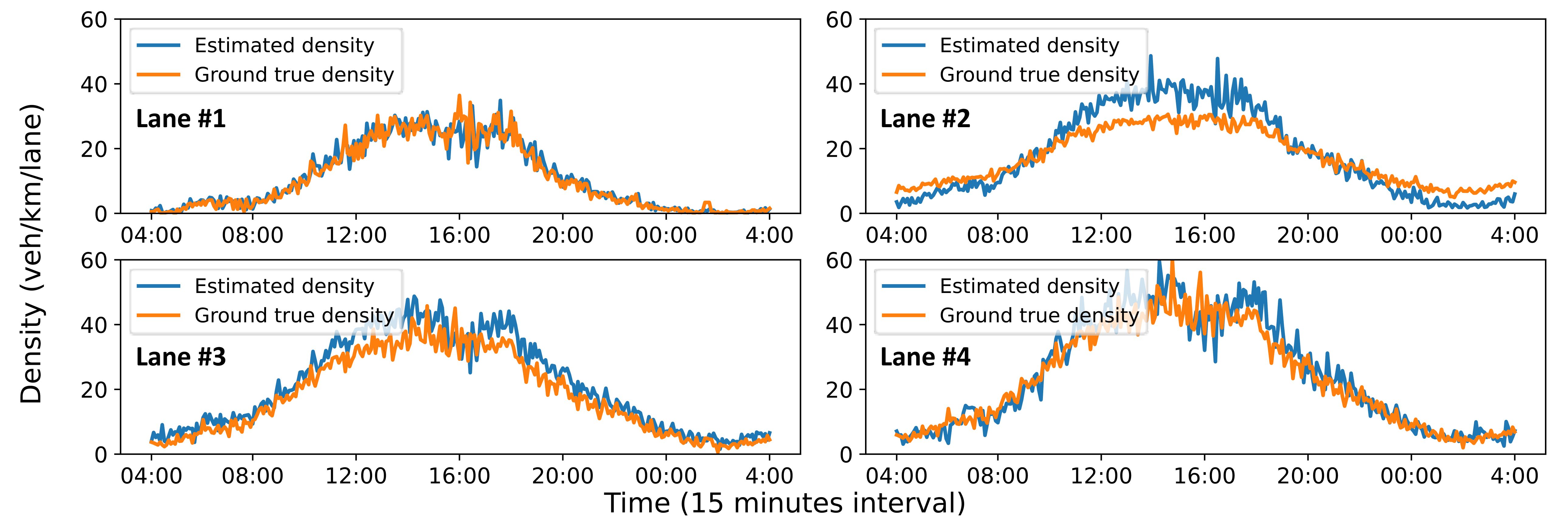}
    \caption{Traffic density from 4:00 AM to 4:00 AM on the next day in \texttt{Sac}.}
    \label{fig:traffic_density_cal}
\end{figure}
\begin{table}[h]
    \begin{tabular}{p{0.20\columnwidth}|p{0.23\columnwidth}p{0.23\columnwidth}p{0.23\columnwidth}}
    \hline
    \textbf{Lane ID} & \textbf{RMSE} & \textbf{MAE} & \textbf{MAPE} \\
    \hline
    Lane \#1 & 1.13 & 0.73 & 34.54\% \\
    Lane \#2 & 2.14 & 1.69 & 28.45\% \\
    Lane \#3 & 1.83 & 1.43 & 30.02\% \\
    Lane \#4 & 1.87 & 1.37 & 18.46\% \\
    Average  & 1.33 & 1.30 & 27.87\% \\
    \hline
    \end{tabular}
    \caption{The RMSE, MAE and MAPE of the estimated density in different lanes from monitoring cameras in \texttt{Sac} (unit for RMSE, MAE: veh/km/lane)}
    \label{tab:Sacramento_density}
\end{table}
it can been seen that the estimated density in \texttt{Sac} is close to the ground truth. The MAE for Lane \#1 is only approximately 0.73 veh/km/lane, while those for the other three lanes are approximately 1.50 veh/km/lane. The RMSE and MAPE of all lanes are acceptable, meaning that the method can accurately capture the variation in traffic density. 

\subsection{Spatio-temporal patterns of the density}

The spatio-temporal distribution of the density is shown in Figure~\ref{fig:st_Sacramento}.  The same as the heatmap in \texttt{HK}, where a lighter color represents a higher density. One can see that the traffic density starts to grow at 10:00 AM, and the peak period is approximately from 2:00 PM to 5:00 PM at local time. Along the road, the traffic density is uniformly distributed between downstream and upstream sections of the road. 
\begin{figure}[h]
    \centering
    \includegraphics[width=1\textwidth]{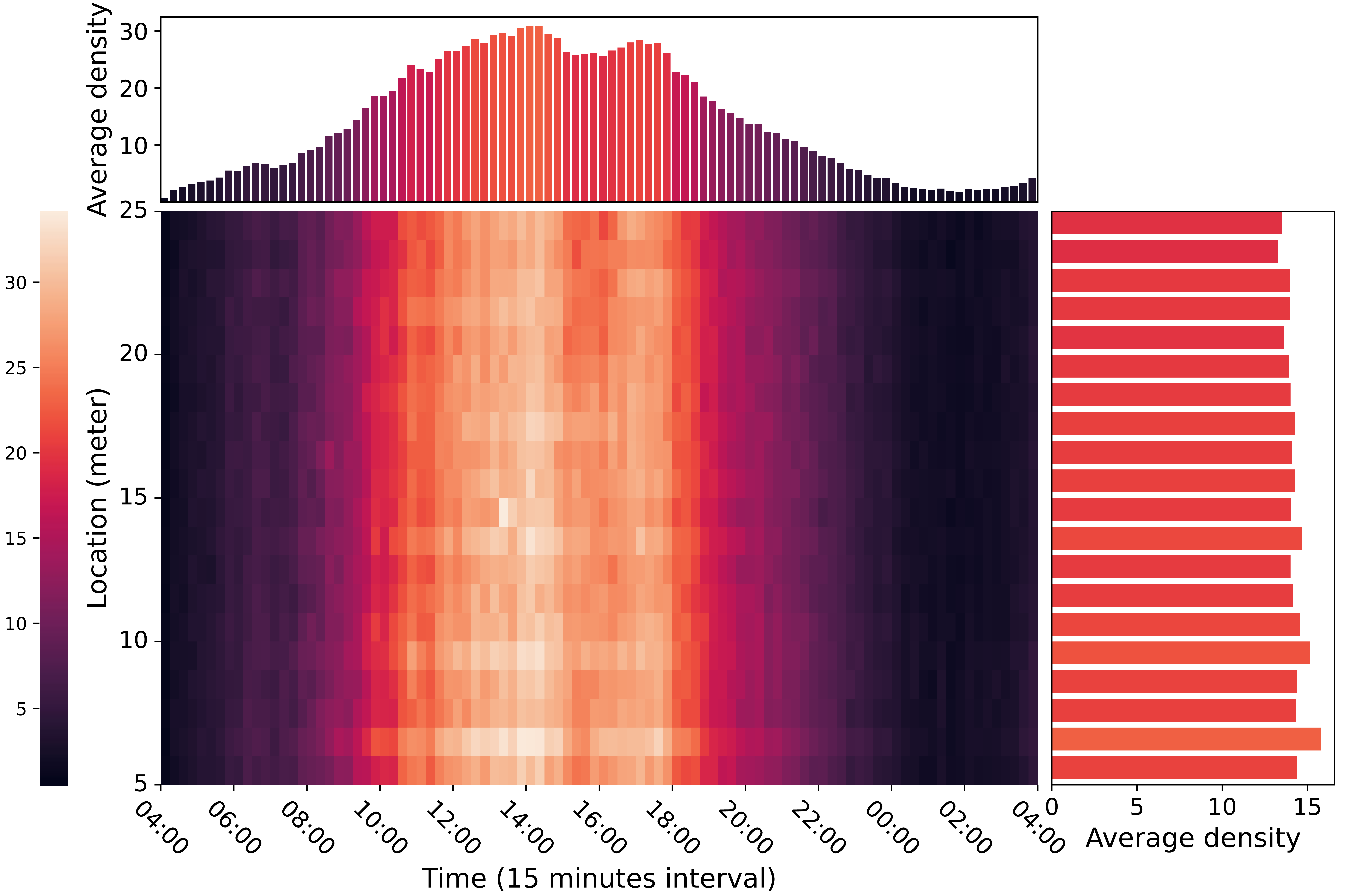}
    \caption{The spatio-temporal distribution of traffic density on an averaged day (unit:veh/km/lane) in \texttt{Sac}.}
    \label{fig:st_Sacramento}
\end{figure}

\begin{figure}[h]
    \centering
    \includegraphics[width=1\textwidth]{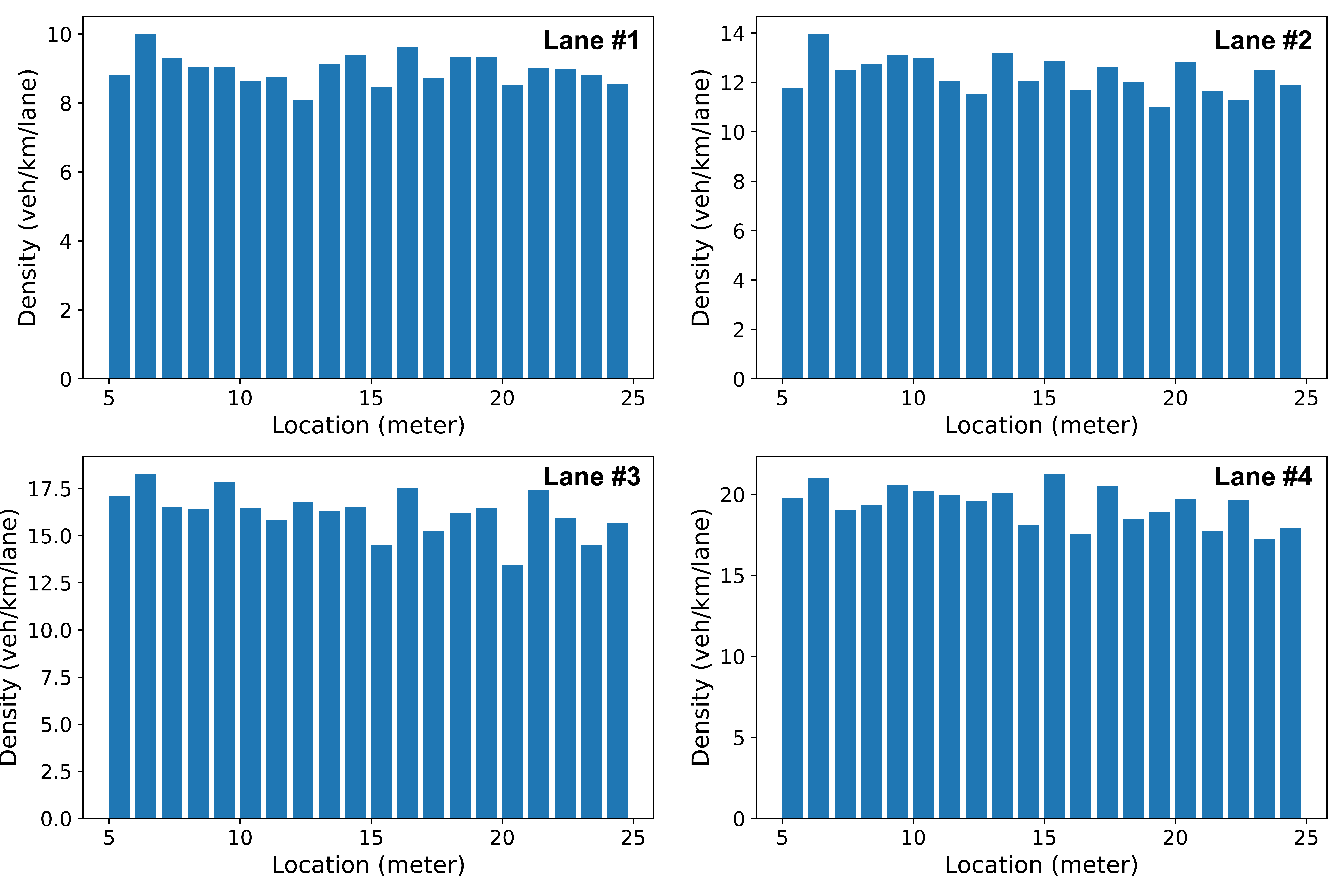}
    \caption{Average density along the road for each lane in \texttt{Sac}.}
    \label{fig:location_hist_Sacramento}
\end{figure}
Spatial distribution of density along the road for each lane is shown in Figure~\ref{fig:location_hist_Sacramento}. One can see that vehicles are uniformly distributed on each lane, and the average densities are different. Overall, the variation in the density in \texttt{Sac} is much smaller than that in \texttt{HK}, and the corresponding accuracy in density estimation is also better in \texttt{Sac}. It is reasonable to conjecture that the congested and dynamic road conditions can deteriorate the accuracy of the density estimation. 

\subsection{Calibration of the fundamental diagrams}

Incorporating the traffic speed data for each lane from the loop detectors in the PeMS, we also calibrate the fundamental diagram for each lane. Owing to the limited number of data points, calibration with the Newell model \citep{newell} is not practically feasible. Instead, we use the Greenshields model \citep{greenshield} to fit the relationship between density-speed and density-flow. Figure~\ref{fig:fd_Sacramento} shows the relationship between density-flow and density-speed respectively.
\begin{figure}[h]
    \centering
    \includegraphics[width=1\textwidth]{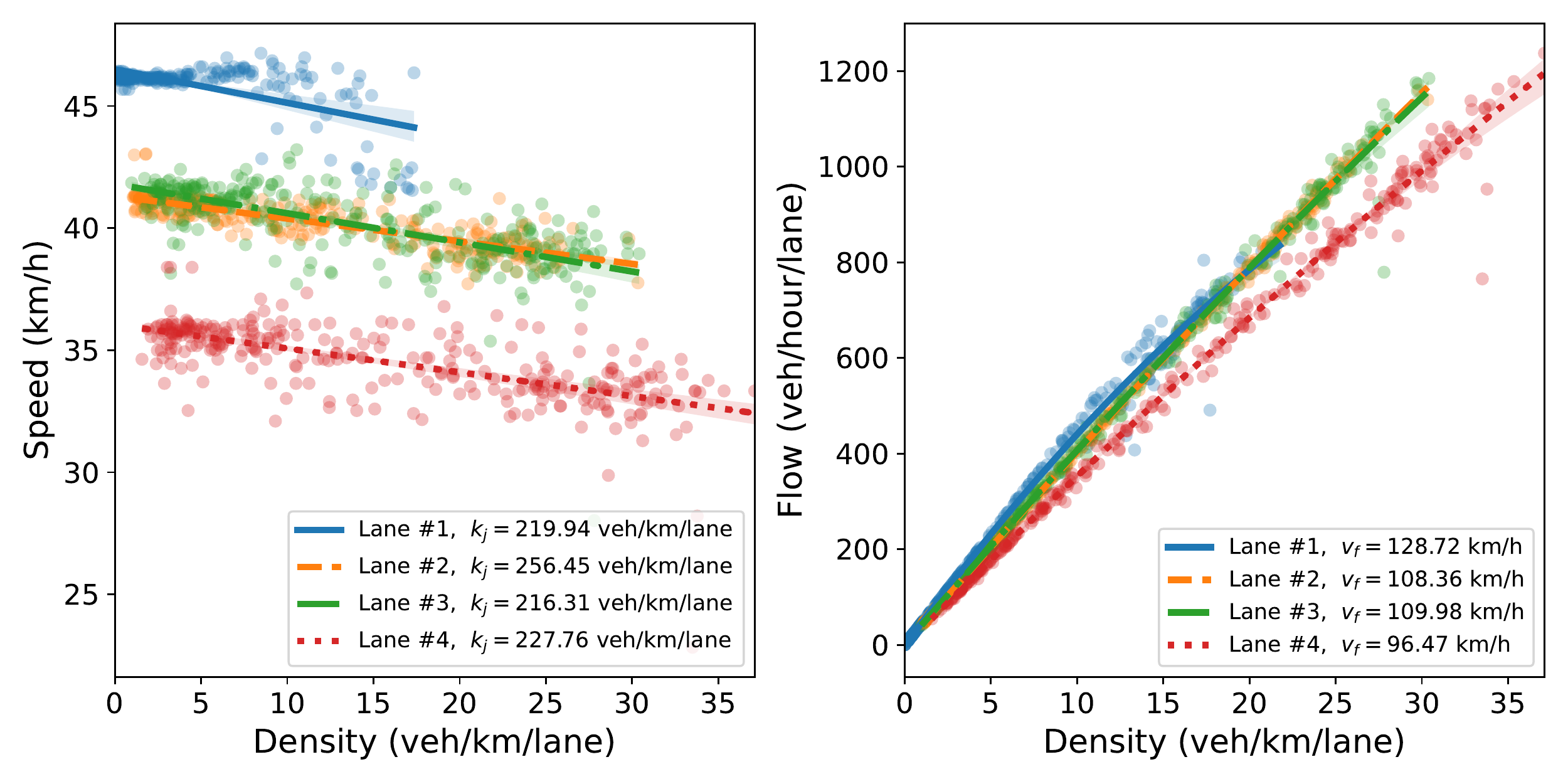}
    \caption{The calibrated fundamental diagrams for each lane in \texttt{Sac}.}
    \label{fig:fd_Sacramento}
\end{figure}
Since the congestion is lighter on Sunday, vehicles are driven close to the free flow speed at most times. Hence the $k_j$ might be overestimated. It can be seen that Lane \#1 possesses the highest free flow speed indicating that vehicles usually drive fastest on the leftmost lane. Lanes \#2 and \#3 share similar free flow speeds, while the free flow speed on Lane \#4 is smaller. A probable reason is that Lane \#4 is connected to an off-ramp to local roads, and vehicles are prone to decelerate when it approaches the off-ramp. 

\clearpage

\section{Conclusions}
\label{section:conclusion}
In this paper, we propose a holistic framework for traffic density estimation using traffic monitoring cameras with 4L characteristics, and the {\bf4L} represents {\bf L}ow frame rate, {\bf L}ow resolution, {\bf L}ack of annotated data, and {\bf L}ocated in complex road environments. The proposed density estimation framework consists of two major components: camera calibration and vehicle detection. For camera calibration, a multi-vehicle calibration method named \texttt{MVCalib} is developed to estimate the actual length of roads from camera images. For vehicle detection, the transfer learning scheme is adopted to fine-tune the deep-learning-based model parameters. A linear-program-based data mixing strategy that incorporating multiple datasets is proposed to synergize the performance of the vehicle detection model. 

The developed camera calibration and vehicle detection models are compared with existing baseline models in terms of the performance on real world monitoring camera data in Hong Kong and Sacramento, and both models outperform the existing the state-of-the-art models. The MAE of camera calibration is less than 0.2 meters out of 6 meters, and the accuracy of the detection model is approximately 90\%. We further conduct two case studies in Hong Kong and Sacramento to evaluate the quality of the estimated density. The experimental results indicate that the MAE for the estimated density is 9.04 veh/lane/km in Hong Kong and 1.30 veh/lane/km in Sacramento. Comparing the estimation  results in the two study regions, we also observe that the performance of the proposed density estimation framework degrades under congested traffic conditions.

By demonstrating the effectiveness of the proposed density estimation framework in two study regions in different countries, we validate the framework's potential for large-scale traffic density estimation from monitoring cameras in cities across the globe, and the proposed framework could provide considerable and fruitful information for traffic operations and management applications without upgrading the hardware.

In the future research, we would like to extend the proposed framework to estimate other traffic state variables such as speed, flow, and occupancy. In the camera calibration method, the key points of each vehicle are manually labeled, which can be further automated \citep{Bhardwaj2018autocalib}. In addition to the vehicle detection model, a vehicle classification model could also be developed to estimate the traffic density by vehicle type \citep{car_class}. It would be interesting to explore the domain adaptation approach to detect vehicles in various traffic scenarios \citep{veh_detect_da_1,veh_detect_da_2}. Moreover, it would be of practical value to develop a fully automated and end-to-end pipeline to deploy the proposed density estimation framework in different traffic surveillance systems.

\section*{Supplementary Materials}
The source codes for the developed camera calibration models as well as the trained vehicle detect model in the proposed traffic density estimation framework can be found at GitHub\footnote{\url{https://github.com/ZijianHu/Traffic_Density_Estimation}}.

\section*{Acknowledgments}
The work described in this study was supported by grants from the Research Grants Council of the Hong Kong Special Administrative Region, China (Project No. PolyU R5029-18 and R7027-18) and a grant from the Research Institute for Sustainable Urban Development (RISUD) at the Hong Kong Polytechnic University (Project No. P0038288). The contents of this paper reflect the views of the authors, who are responsible for the facts and the accuracy of the information presented herein. 

\bibliography{citation}

\clearpage
\appendix

\section{Notations}
\label{apx:notations}
\begin{longtable}{p{0.15\columnwidth}|p{0.80\columnwidth}}
\hline
\textbf{Variables} & \textbf{Definitions}  \\
\hline
\multicolumn{2}{l}{\textbf{Traffic-related Variables}} \\
\hline
$k$ & Traffic density. \\
$k_j$ & Jam density. \\
$L$ & Length of the road in the study region. \\
$N$ & Number of vehicles in the study region. \\
$\overline{v}_f$ & Free-flow speed. \\
\hline
\multicolumn{2}{l}{\textbf{Camera Calibration}} \\
\hline
$\hat{C}_i$ & The centroid of all the back-projected key points on the $i$th vehicle. \\ 
$\hat{f}$ & Default focal length. \\
$f^i$ & The focal length estimated with the $i$th vehicle in the vehicle model matching stage. \\
$f_j^i$ & The focal length estimated with the $i$th vehicle under the $j$th model in the candidate generation stage. \\
$h$ & Height of the camera image. \\
$K$ & The matrix of endogenous parameters of the traffic monitoring camera. \\
$m$ & Number of potential vehicle models. \\
$M_i$ & Number of key points on the $i$th vehicle. \\
$n$ & Number of vehicle in the traffic monitoring images. \\
$\mathcal{p}^i$ & The set of two-dimensional key points of the $i$th vehicle. \\
$\mathcal{P}_j^i$ & The sets of three-dimensional key points of the $i$th vehicle in real world presumed that the vehicle model is $j$. \\
$p_k^i$ & The coordinate of the $k$th two-dimensional key points on the $i$th vehicle.\\
$P_{j,k}^i$ & The coordinates of the $k$th two-dimensional key points on the $i$th vehicle given the vehicle model of $j$. \\
$R^i$ & The rotation matrix estimated with the $i$th vehicle in the vehicle model matching stage. \\
$R_j^i$ & The rotation matrix estimated with the $i$th vehicle under the $j$th model in the candidate generation stage. \\
$R_j^i\Big\vert_k$ & The $k$th column of the rotation matrix estimated with the $i$th vehicle under the $j$th model in the candidate generation stage. \\
$T^i$ & The translation vector estimated with the $i$th vehicle in the vehicle model matching stage. \\
$T^{i'}$ & The translation vector estimated with the anchor vehicle $i'$ in the vehicle model matching stage. \\
$T_j^i$ & The translation vector estimated with the $i$th vehicle under the $j$th model in the candidate generation stage. \\
$T_j^i \Big\vert_k$ & The $k$th column of the translation vector estimated with the $i$th vehicle under the $j$th model in the candidate generation stage.\\
$\alpha$ & A hyper-parameter adjusting the weight distance loss and angle loss. \\
$w$ & The width of the camera image. \\
$\psi^i$ & The estimated parameters including $\{f^i, R^i, T^i\}$ of the $i$th vehicle in the vehicle model matching stage \\
$\psi_j^i$ & The estimated parameters including $\{f^i_j, R^i_j, T^i_j \}$ of the $i$th vehicle given the $j$th model in the candidate generation stage. \\
$\tau$ & The hyper-parameter of controlling the distribution of the weighting function. \\
$\omega(\hat{C}_i, \hat{C}_{i'} \vert \tau)$ & The weighting function for the $i$th vehicle using the anchor vehicle $i'$. \\
$\xi_l(i, k_1, k_2, \psi^{i'})$ & The distance loss of the $k_1$th and $k_2$th key points on the $i$th vehicle given the parameters estimated from anchor vehicle $i'$. \\
$\xi_r(i, k_1, k_2, \psi^{i'})$ & The angle loss of the $k_1$th and $k_2$th key points on the $i$th vehicle given the parameters estimated from anchor vehicle $i'$. \\
$L_f(\psi^{i'} \vert \alpha, \tau)$ & The weighted objective function for optimizing the parameters from anchor vehicle $i'$ in the parameter fine-tuning stage. \\
$L_p(i, \psi^{i'} \vert \alpha)$ & The back-projection loss of the $i$th vehicle given the parameter from anchor vehicle $i'$ from the image to the real-world including distance loss and angle loss.\\
$L_v(\psi_j^i)$ & The projection loss from the real world to image given the parameters estimated from the $i$th vehicle given the model of $j$. \\
$\hat{P}_k^i(\psi^{i'})$ & The back-projected point on the $i$th vehicle of the $k$th key point under the camera parameter of the anchor vehicle $i'$. \\
\hline
\multicolumn{2}{l}{\textbf{Vehicle Detection}} \\
\hline
$u$ & Number of datasets. \\
$v$ & Number of traffic scenarios. \\
$q_{\mu, \nu}$ & Number of images that will be incorporated in the LP hybrid dataset from dataset $\mu$ for the scenarios $\nu$. \\
$Q_{\mu, \nu}$ & Number of images in dataset $\mu$ for the scenario $\nu$. \\
$\beta$ & The maximum tolerance for the upper and lower bound of image number in different traffic scenarios. \\
$\gamma$ & The maximum tolerance balancing the image number from different datasets. \\ \hline 
\end{longtable}

\section{The Interpretation of Metrics of Vehicle Detection}
\label{apx:explain_metrics}
The metrics for evaluating the accuracy of vehicle detection models include precision, recall, PR-curve, mAP@0.5, and mAP@0.5:0.95. These metrics are commonly used to evaluate the quality of object detection models in CV. Before introducing the concept of the above metrics, there is a prerequisite metric called intersection over union (IoU), which defines the gaps between the estimated objection location and the ground truth. The outputs of the detection model are two-fold. One is four corner coordinates that locates the object position in the image. The other is the confidence probability of the belonging category. If we overlap the estimated and the ground true bounding boxes, there will be an area of intersection (shown in Figure~\ref{fig:aoi}) and an area of union (shown in Figure~\ref{fig:aou}), where the red and green rectangle means the estimated and ground true bounding boxes of an object, and the blue rectangle shows the intersection and union area, respectively. The intersection over union is defined as the quotient of the intersection area over the union area.

\begin{figure}[h]
    \centering 
    \begin{subfigure}[b]{0.48\textwidth}
        \includegraphics[width=\textwidth]{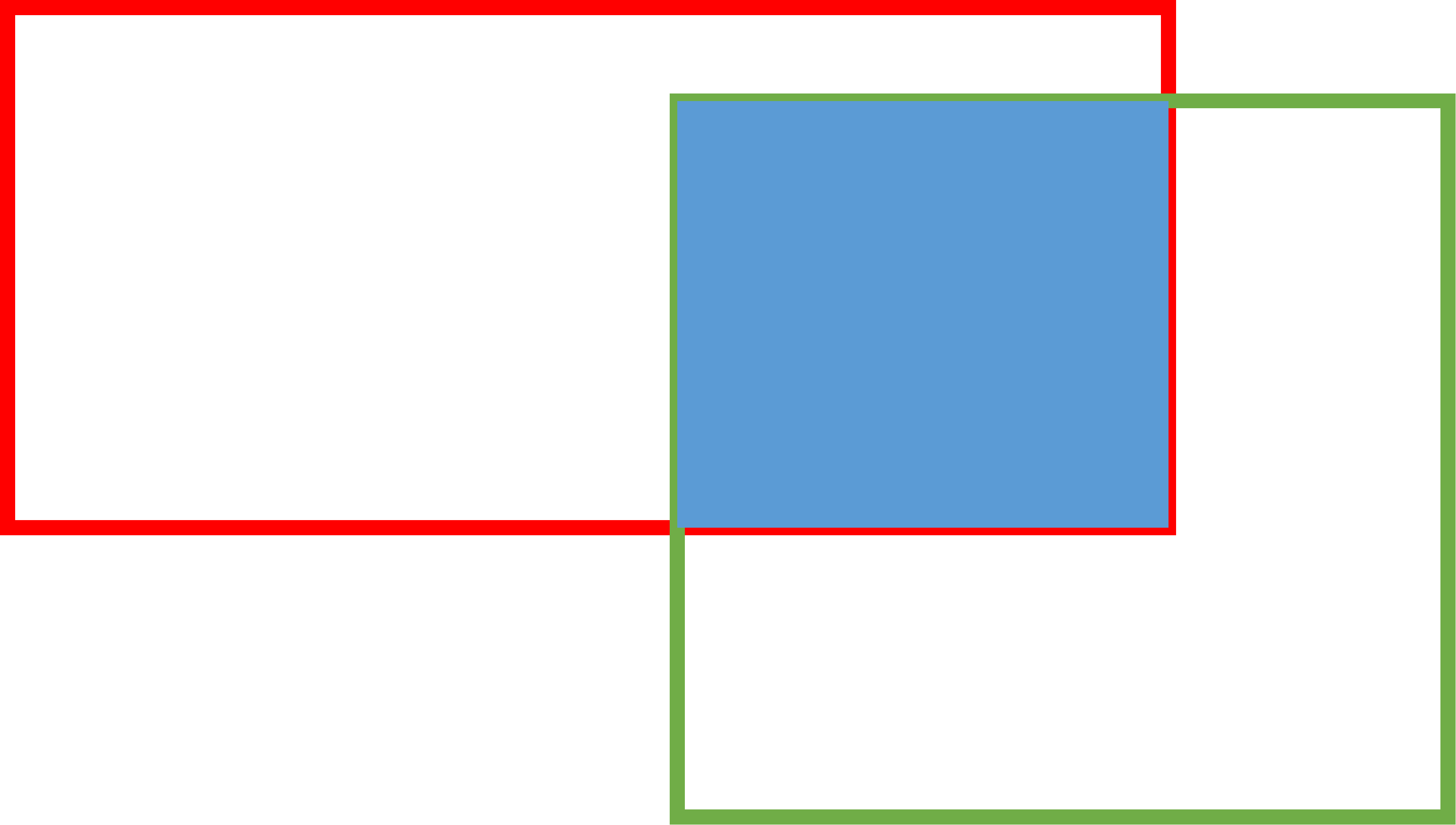}
        \caption{Area of intersecion.}
        \label{fig:aoi}
    \end{subfigure}
    \begin{subfigure}[b]{0.48\textwidth}
        \includegraphics[width=\textwidth]{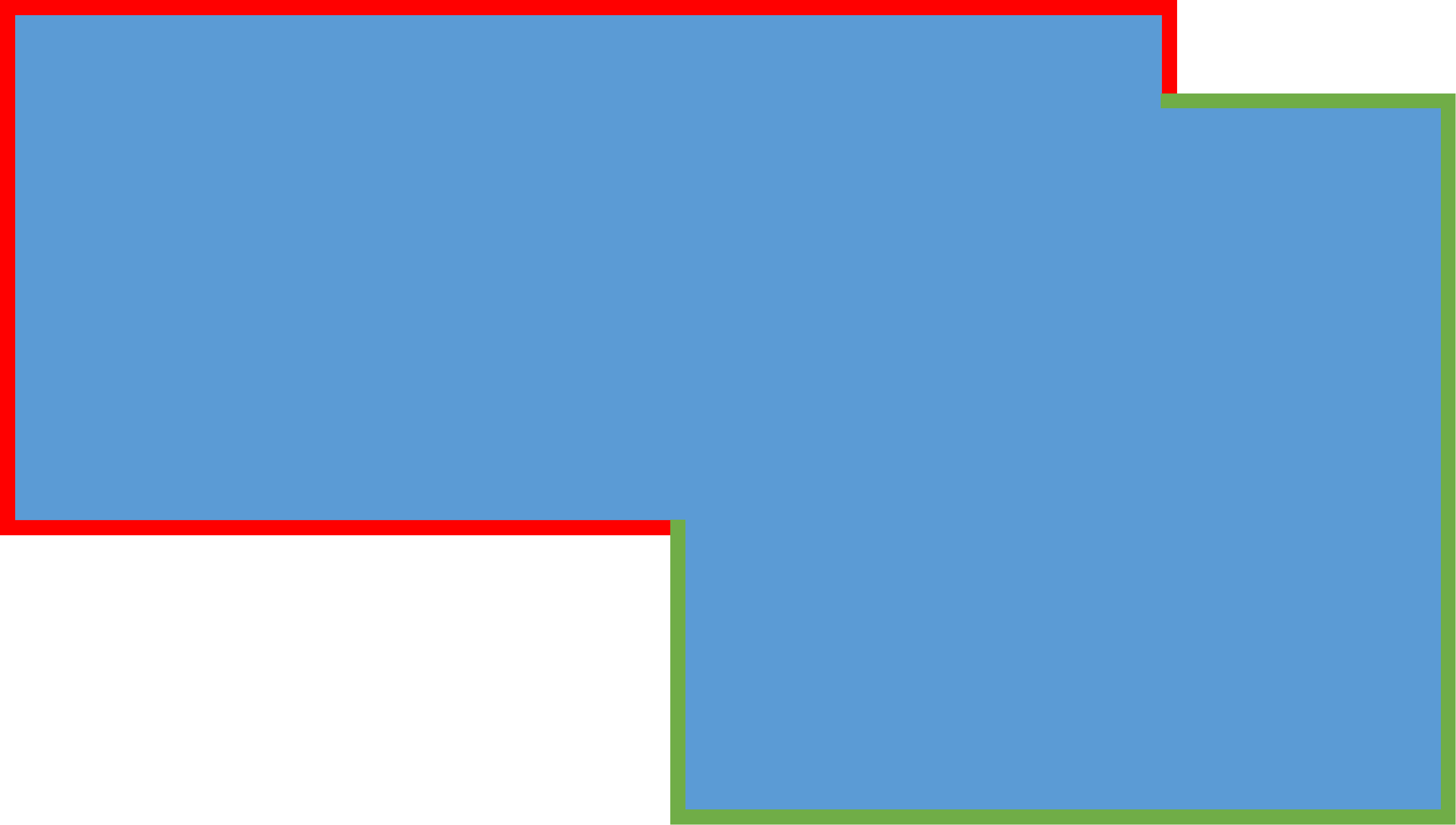}
        \caption{Area of union.}
        \label{fig:aou}
    \end{subfigure}
    \caption{Illstruation of intersecion over union (IoU).} 
    \label{fig:IoU} 
\end{figure}

A threshold for IoU is set to decide if the bounding box is real or fake. If the IoU exceeds the threshold, we label it as True Positive (TP). Moreover, we can divide all circumstances into three categories, True Positive (TP), False Positive (FP), and False Negative (FN). The illustration about these circumstances is shown in Table~\ref{tab:TP_FP_FN}. 
\begin{table}[h]
    \begin{tabular}{l|l}
    \hline
    \textbf{Categories} & \textbf{Comments} \\
    \hline
    True Positive (TP)  & The IoU between predicted and ground truth exceeds the threshold. \\
    \hline
    False Positive (FP) & \begin{tabular}[c]{@{}l@{}}1. The IoU between predicted and ground truth is smaller than the threshold.\\ 2. Estimated bounding boxes not overlapping with any ground true bounding boxes.\end{tabular} \\
    \hline
    False Negative (FN) & The object is not detected by the algorithm. \\
    \hline
    \end{tabular}
    \caption{Illustration of TP, FP, and FN.}
    \label{tab:TP_FP_FN}
\end{table}

Additionally, the precision and recall can be calculated as

\begin{equation}
    \begin{aligned}
        Precision &= \frac{TP}{TP + FP} \\
        Recall &= \frac{TP}{TP + FN}
    \end{aligned},
\end{equation}

The precision and recall are a a pair of contradictory metric. When the precision is high, the recall is relative low, vice versa. If we rank all the detection results according to the confidence probability, set different thresholds for confidence probability and re-calculate the precision and recall, a precision-recall (PR) curve can be plotted where the x-axis is the recall and the y-axis is the precision. An example is shown in Figure~\ref{fig:PR_curve_HK}. 
The Precision-Recall curves of different vehicle detection models are shown in Figure~\ref{fig:PR_curve_HK}. 
\begin{figure}[h]
    \centering
    \begin{subfigure}[b]{0.32\textwidth}
        \includegraphics[width=\textwidth]{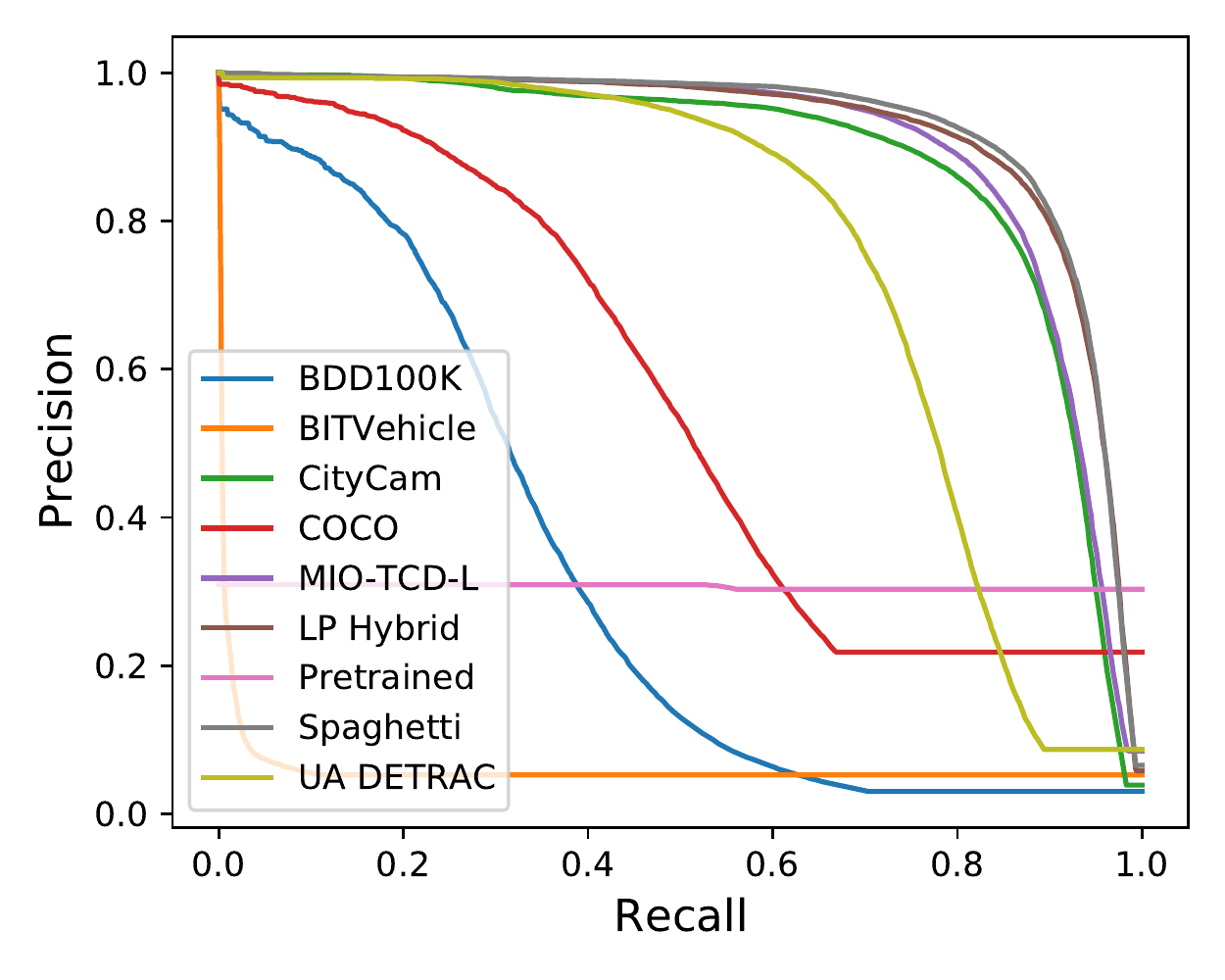}
        \caption{The PR curve for images at daytime.}
        \label{fig:PR_curve_day}
    \end{subfigure}
    \begin{subfigure}[b]{0.32\textwidth}
        \includegraphics[width=\textwidth]{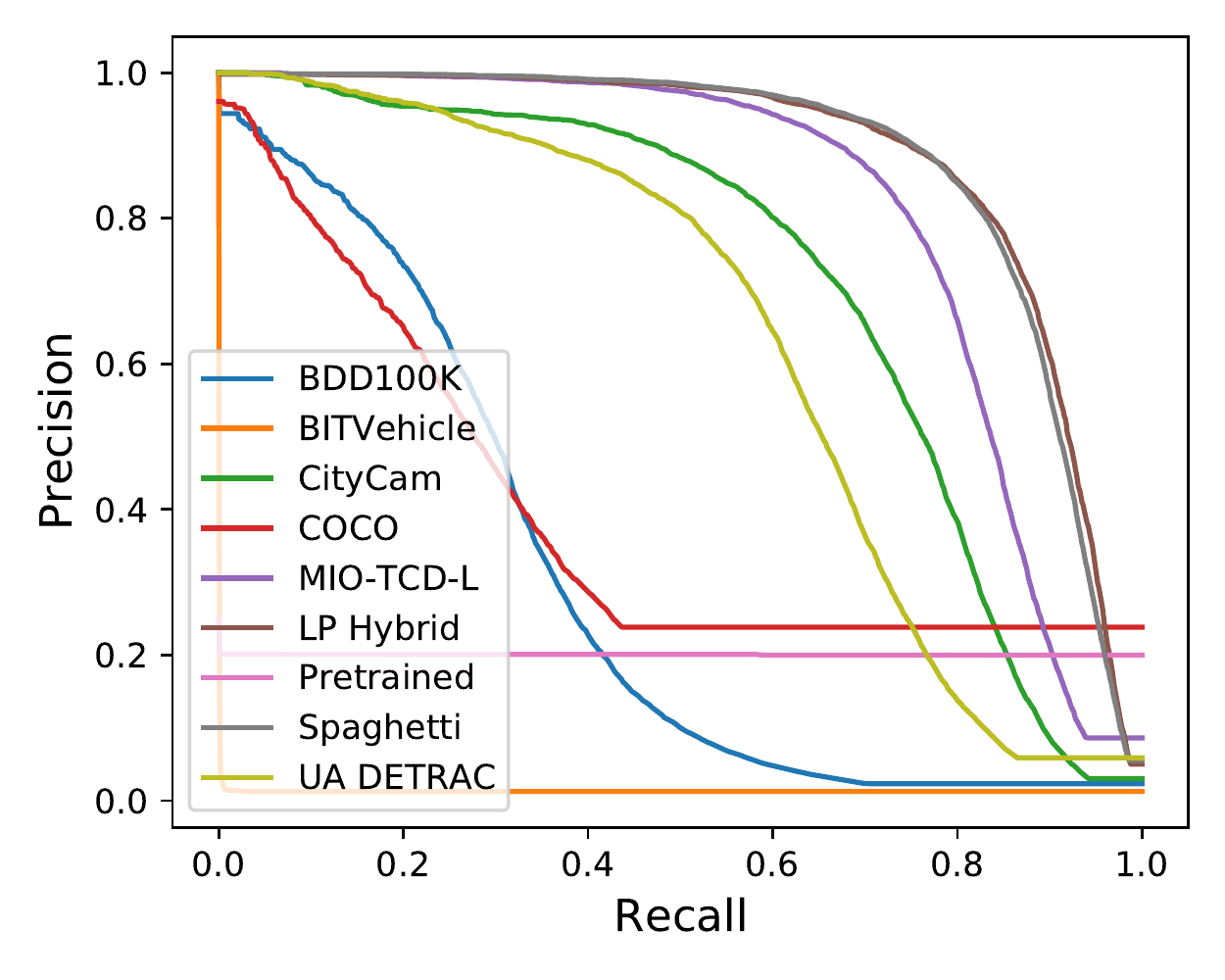}
        \caption{The PR curve for images at nighttime.}
        \label{fig:PR_curve_night}
    \end{subfigure}
    \begin{subfigure}[b]{0.32\textwidth}
        \includegraphics[width=\textwidth]{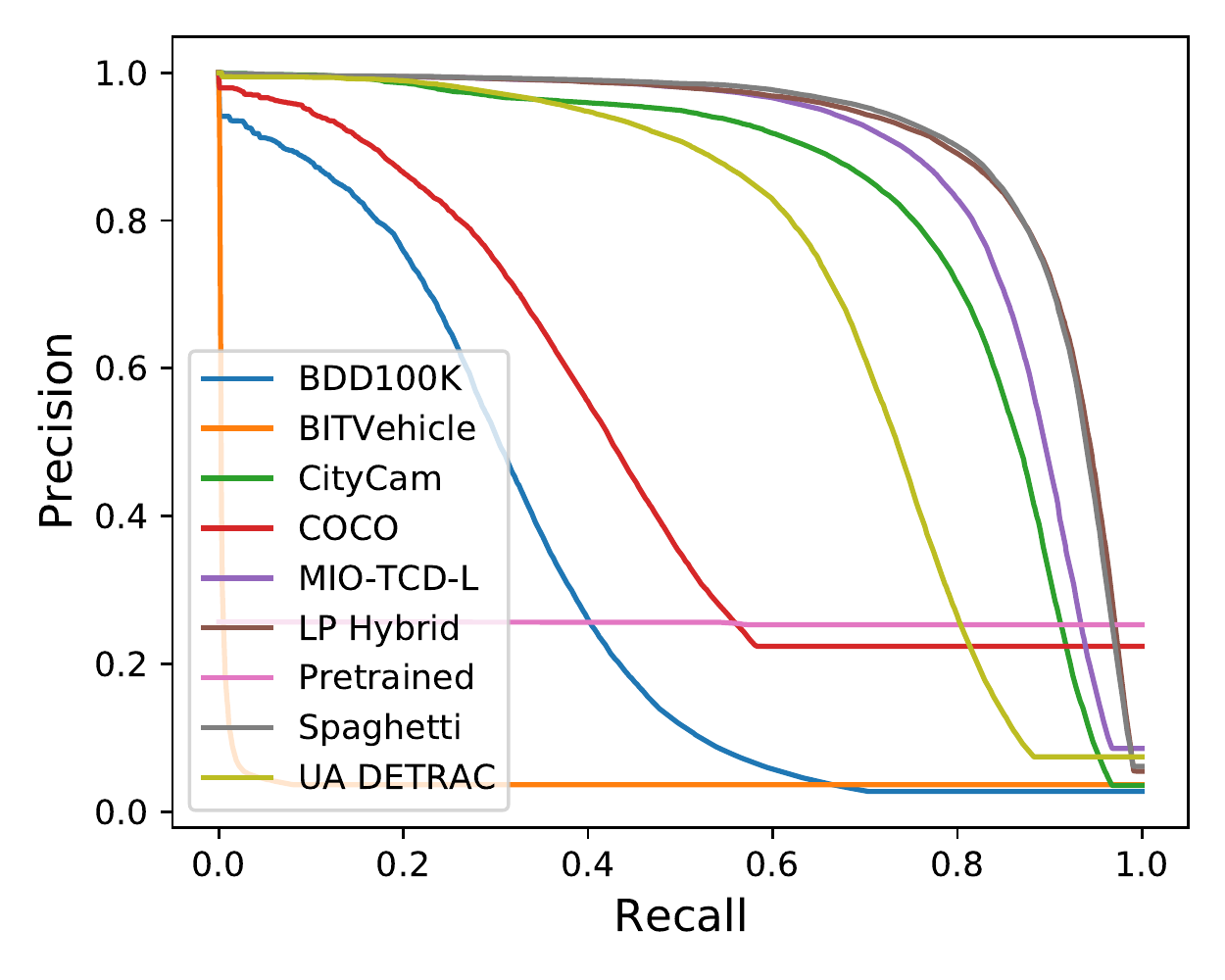}
        \caption{The PR curve for images throughout day and night.}
        \label{fig:PR_curve_full}
    \end{subfigure}
    \caption{The PR curve for camera images on the testing set.}
    \label{fig:PR_curve_HK}
\end{figure}
With the increasing of confidence threshold, the recall enlarges while the precision reduces. If the curve is close to the upper right corner of the figure, the performance of the model is good. Hence, it can be seen that the detection models trained with Spaghetti and LP hybrid datasets outperform other models trained with sole datasets. In particular, if we compare the curve between models trained with Spaghetti and LP hybrid datasets, the differences between these two models are marginal.

The Average Precision (AP) is the area that below the PR curve, calculated as

\begin{equation}
    AP = \int_0^1 PR(r)dr,
\end{equation}
\noindent where $r$ is the recall and $PR(r)$ is the precision. The mAP@0.5 means the AP value when the IoU threshold is 0.5. Besides, the mAP@0.5:0.95 means the average of AP when the IoU threshold equals to $0.5, 0.55, 0.9, \cdots, 0.9, 0.95$ separately. These two metrics are extensive used to evaluate the performance of algorithms in object detection tasks in CV.

\section{Proof of Proposition~\ref{prop:reduce}}
\label{ap:reduce}
Given a unit vector of the rotation axis $\bm{d} \in \mathbb{R}^{3}$, and a rotation process along the rotation axis with a clockwise degree of $\theta$ exist in the three-dimensional space, the rotation matrix can be acquired from the Rodrigues' formula \citep{Gao2017SLAM} as :

\begin{equation}
    R = \cos \theta \bm{I} + \left( 1 - \cos \theta \right) \bm{d} \bm{d}^T + \sin \theta \bm{d} \times, \label{eq:rodrigures}
\end{equation}
where $\bm{I}$ is the identity matrix and 
$\bm{d} \times = \begin{bmatrix} d_1 \\ d_2 \\ d_3 \end{bmatrix} \times = 
\begin{bmatrix} 
    0 & -d_3 & d_2 \\
    d_3 & 0 & -d_1 \\
    -d_2 & d_1 & 0 
\end{bmatrix}$
represents the antisymmetric matrix of $\bm{d}$. 

Taking the trace from the both sides of Equation \ref{eq:rodrigures}:

\begin{equation}
    \begin{aligned}
        \tr \left( R \right) &= \cos \theta \tr \left( \bm{I} \right) + 
        \left( 1 - \cos \theta \right) \tr \left( \bm{d}\bm{d}^T \right) +
        \sin \theta \tr \left( \bm{n} \times \right) \\
        &= 3 \cos \theta + (1 - \cos \theta) \\
        &= 1 + 2 \cos \theta
    \end{aligned},
\end{equation}

\noindent $\theta$ can be solved as:

\begin{equation}
    \theta = \arccos \frac{\tr \left( R \right) - 1}{2}.
\end{equation}

For rotation axis $\bm{d}$, since the location of point on the rotation axis will not change after rotation, this can be solved as follows:

\begin{equation}
    R\bm{d} = \bm{d} \label{rot_invar},
\end{equation}
\noindent Hence the rotation can be compressed into a vector with three elements.
\begin{equation}
    \begin{aligned}
        R \Leftrightarrow \Theta &= \theta \bm{d} \\
        \Theta & \in \mathbb{R}^{3}
    \end{aligned}
\end{equation}
\noindent Therefore, the new parameter space is $f, \Theta, T \in \mathbb{R}^{7}$.
\end{document}